\documentclass{colt2012} 

\usepackage{algorithm,algorithmic}
\usepackage{amssymb, multirow, paralist, color}
\newtheorem{thm}{Theorem}
\newtheorem{prop}{Proposition}
\newtheorem{cor}[thm]{Corollary}
\newtheorem{ass}{Assumption}
\usepackage{enumitem}

\def \A {\mathcal{A}}
\def \X {\mathcal{X}}
\def \Xb {\mathbf{X}}

\def \R {\mathbb{R}}

\def \w {\mathbf{w}}
\def \v {\mathbf{v}}

\def \x {\mathbf{x}}

\def \E {\mathrm{E}}

\def \x {\mathbf{x}}

\def \b {\mathbf{b}}

\def \z {\mathbf{z}}

\def \y {\mathbf{y}}
\def \u {\mathbf{u}}
\def \H {\mathcal{H}}

\def \F {\mathcal{F}}

\def \xh {\widehat{\x}}
\def \wh {\widehat{\w}}

\def \fh {\widehat f}

\def \D {\mathbb{D}}

\def \Z {\mathcal{Z}}

\usepackage{pifont}

\begin{document}

\title[Non-Convex Truncated Losses]{Learning with Non-Convex Truncated Losses by SGD}
 \author{\Name{Yi Xu}$^{1}$\Email{yi-xu@uiowa.edu}\\ 
 \Name{Shenghuo Zhu}$^{2}$\Email{shenghuo.zhu@alibaba-inc.com}\\
\Name{Sen Yang}$^{2}$\Email{senyang.sy@alibaba-inc.com}\\
\Name{Chi Zhang}$^{2}$\Email{yutou.zc@alibaba-inc.com}\\
\Name{Rong Jin}$^{2}$\Email{jinrong.jr@alibaba-inc.com}\\
\Name{Tianbao Yang}$^{1}$\Email{tianbao-yang@uiowa.edu}\\ 
   \addr$^1$Department of Computer Science, The University of Iowa, Iowa City, IA 52242, USA \\
   \addr$^2$Machine Intelligence Technology, Alibaba Group, Bellevue, WA 98004, USA
}
\maketitle
\vspace*{-0.5in}
\begin{center}{First version: May 20, 2018}\end{center}
\begin{abstract}
Learning with a {\it convex  loss} function has been a dominating paradigm for many years. It remains an interesting question how non-convex loss functions help improve the generalization of learning  with broad applicability. 
In this paper,  we study a family of objective functions  formed by truncating traditional   loss functions, which  is applicable to both shallow learning and deep learning.
 Truncating loss functions has  potential to be less vulnerable and more robust to large noise in observations that could be adversarial. More importantly, it is a generic technique without assuming the knowledge of  noise distribution.   To justify non-convex learning with truncated losses,  we establish excess risk bounds of empirical risk minimization based on truncated losses for heavy-tailed output, and statistical error of an approximate stationary point found by  stochastic gradient descent (SGD) method.  
Our experiments for shallow and deep learning for regression with outliers, corrupted data and heavy-tailed noise further justify the proposed method.
\end{abstract}

\section{Introduction}
A fundamental problem in machine learning can be described as follows. Let $Z=(X, Y)\sim\mathbb D$ denote a random data following an unknown  distribution of $\D$, where $X\in\X\subseteq \R^d$ denotes a random input and $Y\in\mathcal Y\subseteq\R$ denotes its corresponding output. Let $\mathcal H=\{h: \mathcal X\rightarrow\mathcal Y\}$ denote a hypothesis class and $\ell(\cdot, Y)$ denote a loss function. Given a set of training data $\{(\x_i, y_i), i=1,\ldots, n\}$, the problem is to find a hypothesis $h_n\in\mathcal H$ close to a hypothesis that minimizes the expected risk $P(h):=\E_{Z}[\ell(h(X), Y)]$. A state-of-the-art approach is empirical risk minimization (ERM):
\begin{align}\label{prob2}
h_n=\arg\min_{h \in \mathcal H}P_n(h) := \frac{1}{n}\sum_{i=1}^{n}\ell(h(\x_i), y_i)).
\end{align}
For large-scale problems with large $n$, the above problem could be efficiently solved by stochastic algorithms, e.g., stochastic gradient descent (SGD) method~\citep{bottou-2010-large}. A central question in learning theory is to characterize how close is the learned hypothesis $h_n$ to the optimal hypothesis $h_*\in\H$ that minimizes $P(h)$. In machine learning community,  one is usually concerned with the {\it excess risk} $P(h_n)-P(h_*)$. In statistics community, one usually assumes a statistical model between $Y$ and $X$, e.g., $Y=h_*(X)+\varepsilon$, where $h_*\in\mathcal H$, and $\varepsilon$ is a zero-mean random noise, and studies the {\it statistical error} $\|h - h_*\|$ measured in some norm. There are extensive results of excess risk bounds and statistical error bounds of ERM based on a convex loss function (e.g., logistic  loss, square loss)  when data $(X, Y)$ and noise $\varepsilon$ have sub-Gaussian tails (e.g., Gaussian, bounded support~\citep{bartlett2006empirical, boucheron2005theory, massart2007concentration, van2000applications, koltchinskii2011, lecue2013learning, mehta2014stochastic, pmlr-v65-zhang17a}). However, when distribution of data or noise deviates from sub-Gaussian, minimizing the standard convex loss functions might yield poor performance~\citep{brownlees2015empirical}. 

Previous works for handling this issue  either suffer from requiring strong  knowledge of deviation or has high computational costs (see Related Work). In practice, it is rarely the case that the knowledge of data abnormality is given a-prior. Thus, the methods assuming these knowledge are not applicable. In this paper, we consider a generic method by  minimizing truncated losses. The intuition is that if a particular data $(X_i, Y_i, \varepsilon_i)$ deviates from normal behavior, the loss $\ell(h(X_i), Y_i)$ could be very large and therefore can be truncated to mitigate its effect on misleading the learning process. In particular, we consider a family of non-convex truncated function $\phi(\ell)$ with varied truncation level, and minimize the following ERM problem with truncated loss:
\begin{align}\label{prob2}
\hat h_n=\arg\min_{h \in \mathcal H}\hat P_n(h) := \frac{1}{n}\sum_{i=1}^{n}\phi(\ell(h(\x_i), y_i))).
\end{align}
There are several noticeable merits of this method: (i) the truncation can be used with any standard convex loss functions (e.g., square loss, absolute loss); (ii) the problem is still of a finite-sum form which enables one to employ simple SGD to solve it; (iii) it does not depend on knowledge of abnormality. Although minimizing truncated losses has been considered and adopted by practitioners~\citep{DBLP:conf/iccv/BelagiannisRCN15},  several challenging questions have not been well addressed: (i) what is the excess risk round  of $\hat h_n$ under abnormality of data; (ii) how to quantitively understand the benefit of truncation;  (iv) if SGD is stuck at stationary points, what can be said about the performance of stationary points; (v) if SGD is employed to solve~(\ref{prob2}) with non-smooth loss functions, what convergence guarantee can be established.  Our analysis will revolve around these questions. In particular, our contributions are summarized below: 

\begin{itemize}[leftmargin=*]
\item We establish  excess risk bounds of $\hat h_n$. We show that the empirical minimizer  for minimizing the average of truncated losses can enjoy an excess risk bound of $O(1/\sqrt{n})$ for heavy-tailed $Y$. This result is applicable to both Lipschitz loss functions $\ell(z, y)$ (e.g., absolute loss) and  square loss function, linear models and non-linear models.  
\item  For learning a linear model by truncating square losses, we establish a statistical error bound  of an approximate stationary point found by SGD  that depends on the noise distribution.  We quantitively analyze the benefit of truncation. In particular, our analysis shows that within a certain range of truncation levels, larger truncation could yield smaller statistical error. More importantly, truncation can tolerate much higher noise for enjoying consistency than without truncation. 
\item We  consider the convergence  of SGD for minimizing truncated Lipschitz losses without smoothness assumption. We show that SGD can still converge to points that are close to $\epsilon$-stationary points with an iteration complexity of $O(1/\epsilon^4)$, which is the same as SGD for minimizing smooth functions.  
\item We conduct experiments for both linear models  and non-linear  deep models for regression with corrupted data  and heavy-tailed noise to  justify the effectiveness of  the considered method. 
\end{itemize}

\section{Related Work} \label{sec:related}
Recent advances have sparked increasing interests in non-convex learning (NCL) (i.e.,  learning with non-convex objective functions and/or constraints). Below we will focus on review of non-convex learning for tackling data abnormality, in particular corruptions in  $Y$ and $X$, heavy-tailed noise $\varepsilon$.
 
Numerous studies have considered corruptions in the output $Y$~\citep{DBLP:journals/tit/NguyenT13,DBLP:conf/nips/Bhatia0KK17,DBLP:conf/nips/BhatiaJK15,DBLP:journals/tit/NguyenT13a,DBLP:conf/nips/DalalyanC12}. A well-studied corruption model is to {\it assume that $\y=\mathbf X\w_* + \varepsilon + \b\in\R^n$}, where $\Xb=(\x_1, \ldots, \x_n)^{\top}$,  $\varepsilon =(\varepsilon_1, \ldots, \varepsilon_n)^{\top}$ are sub-Gaussian noises, and $\b=(b_1, \ldots, b_n)^{\top}$ is a sparse vector with non-zero components corresponding to corrupted outputs. Recently, \cite{DBLP:conf/nips/Bhatia0KK17,DBLP:conf/nips/BhatiaJK15} have studied minimizing a non-convex problem for recovering $\w_*$ for sub-Gaussian inputs $\x_i$. For example, the method proposed in~\citep{DBLP:conf/nips/Bhatia0KK17} based on iterative hard-thresholding  is motivated by solving a non-convex problem $\min_{\w, \|\b\|_0\leq k_*}\|\Xb^{\top}\w - \y - \b\|_2^2 = \min_{\|\b\|_0\leq k_*}\|(I- P_X)(\y - \b)\|_2^2$, where $P_X=\Xb(\Xb\Xb^{\top})^{-1}\Xb^{\top}$, where $k_*$ is a assumed sparsity level of $\b$. Consistency of the learned model was proved in~\citep{DBLP:conf/nips/Bhatia0KK17}. 

Several corruption models of input $X$ have been considered~\citep{DBLP:conf/isit/LohW12,DBLP:conf/nips/McWilliamsKLB14}. For example, \cite{DBLP:conf/isit/LohW12} considered  three different corruption models, i.e., additive noise, multiplicative noise, and missing values. They proposed to minimize a non-convex quadratic objective based on estimates of $\Xb\Xb^{\top}\in\R^{d\times d}$ and $\Xb\Xb^{\top}\w_*$ {\it using the knowledge of noise distribution}. The statistical error of the global optimum to the non-convex problem was established and it was also shown that projected gradient descent will converge in polynomial time to a small neighborhood of global minimizers.

The methods mentioned above could achieve superior performance when the corruption of data indeed follows the assumed model. However, in practice it is usually not clear how data is corrupted. A weaker assumption is to consider that the distribution of $X$ or $Y$ or $\varepsilon$ is heavy-tailed with bounded moments.  Several approaches with excess risk guarantee  have been developed based on two popular mean estimators for heavy-tailed data, namely Catoni's mean estimator~\citep{catoni2012challenging, audibert2011robust} and median-of-means estimator~\citep{opac-b1091338, alon1999space}. \cite{brownlees2015empirical} learn a hypothesis based on minimizing Catoni's mean estimator $\hat\mu(h)$,  i.e.,
\vspace*{-0.1in}\begin{align}\label{eqn:catoni}
 \min_{h\in\H}\hat\mu(h), \quad s.t.\quad \frac{\alpha}{n}\sum_{i=1}^n\phi((\ell(h(\x_i), y_i)-\hat\mu(h))/\alpha)=0,
\end{align}
where $\alpha>0$ is a parameter and $\phi(\cdot) = sign(x)\log(1+|x| + x^2/2)$. They established $O(1/\sqrt{n})$ excess risk bound. However, {\it their method is computationally expensive}. In particular, it needs to find  a scalar $\hat\mu(h)$   that satisfies the equality in~(\ref{eqn:catoni}) given a $h\in\mathcal H$, then to search for $h$ that minimizes $\hat\mu(h)$. Although SGD can be used for the root finding problem~(\ref{eqn:catoni}), the minimization of $\hat\mu(h)$ does not have a nice structure to allow for an efficient solver. Some studies have provided efficient algorithms based on different estimators for learning with heavy-tailed data~\citep{hsu2014heavy, hsu2016loss}. But their results are only applicable in restrictive settings (e.g., for smooth and strongly convex losses), which preclude learning with non-convex objectives (e.g., deep learning). 

\cite{audibert2011robust} proposed a method for learning a linear model  based on solving a non-convex min-max problem and proved  an excess risk bound of $O(1/n)$ for heavy-tailed data with a bounded fourth-order moment for noise and a bounded fourth-order moment for input. They proposed {\it a polynomial time} algorithm based heuristics to solve the non-convex min-max problem. However, it is unclear whether the approximate solution found by the heuristics-based approach satisfy the claimed excess risk bound.  There also exist a bulk of studies focusing on understanding the excess risk bound of (regularized) ERM under certain conditions for unbounded loss (e.g., small ball condition, Bernstein condition,  $v$-central condition, etc.) or in restricted settings (e.g., linear least-squares regression)~\citep{audibert2011robust,cortes2013relative,mendelson2014learning,liang2015learning,lecue2012general,Mendelson2017,lecue2016regularization1,lecue2016regularization2,grunwald2016fast,dinh2016fast}. 

Different from these aforementioned studies, this paper focus on understanding the model learned by minimizing truncated losses without prescribing strong assumptions on data corruption. We note that truncating the loss functions is not first considered in this paper. In robust statistics, M-estimators based on non-convex truncated losses have been studied (e.g., Tukey's biweight~\citep{citeulike:903734}, Cauchy loss~\citep{Black1996}). However, conventional  analysis of these estimators is usually restricted to asymptotic consistency  of global minimizers of learning linear models~\citep{doi:10.1080/00401706.2017.1305299}.  In contrast, we provide excess risk bounds for learning general non-linear models as well, which is applicable to deep learning.  The truncation function was also exploited in recent studies through different ways from ERM~\citep{brownlees2015empirical,audibert2011robust}.  However, none of these studies have addressed the computational issues carefully. In contrast, we employ SGD to solve the non-convex truncated losses and analyze the statistical error of a model learned  by SGD. Finally, we note that statistical error was also analyzed for high-dimensional robust M-estimator in~\citep{DBLP:journals/corr/Loh15}. Their analysis focus on understanding the sufficient conditions for robust linear regression such that the statistical error can be established for local stationary points. However, it is still unclear how truncation helps improve the performance of without truncation given that~\cite{audibert2011robust} has established the statistical error of linear least-squares regression without truncation. In contrast, our results are complementary, which not only establish the excess risk bounds for learning non-linear models, but also exhibit that truncation can tolerate much larger noise than without truncation (e.g., it allows noise increase as the number of samples but still maintains consistency).

\section{Non-convex Learning with Truncated Losses}\label{sec:main:soft}
\subsection{Preliminaries and Notations} \label{sec:preli}
We present some notations and preliminaries in this section. 
For simplicity of presentation, we define $\mathcal F=\{Z\rightarrow \ell(h(X), Y): h\in\H\}$ and $\min_{h\in\H}P(h)$ is equivalent to the following problem: 
\begin{align}\label{prob1}
f_*=\arg\min_{f \in \mathcal F} P(f) := \E_{Z \thicksim \D}[f(Z)].
\end{align}

Let $T$ be a (pseudo) metric space and $D$ be a distance metric. An increasing sequence of $(\A_n)$ of partitions of $T$ is said to be admissible if for all $n=0, \ldots, \#\A_n\leq 2^{2^n}$. For any $t\in T$, let $A_n(t)$ be the unique element of $\A_n$ that contains $t$. Denote by $\Delta(A, D)$ the diameter of the set $A\subset T$ under the metric $D$. Define 
\begin{align*}
\gamma_\beta(T, d) = \inf_{\A_n}\sup_{t\in T}\sum_{n\geq 0}2^{n/\beta}\Delta(A_n(t), D), 
\end{align*}
where the infimum is taken over all admissible sequences. It is notable that $\gamma_\beta(\F, D)\leq \int_0^1{\log N(\F,\epsilon, D)}^{1/\beta}d\epsilon$~\citep{talagrand2005generic}.  

We will consider several distance metrics for the class $\F$. 
For $f, g \in \F$, let $d_m(f, g)$, $d_e(f, g)$, and $d_s(f, g)$ be defined as follows: 
\begin{align*}
d_m(f, g) &= \max_{Z\in\Z} |f(Z) - g(Z)|,\quad d_e(f, g)  = (\E[f(Z) - g(Z)]^2)^{1/2},\\
d_s(f, g)& = \left[\frac{1}{n}\sum_{i=1}^n(f(Z_i) - g(Z_i))^2\right]^{1/2}
\end{align*}
Let $N(\F, \epsilon, d)$ be  $\epsilon$-covering number of the class $\F$ under the distance metric $d$, i.e., the minimal cardinality $N$ of any set $\{f_1, \ldots, f_N\}\subset\F$ such that for all $f\in\F$ there exists $f_i\in\{f_1,\ldots, f_N\}$ with $d(f, f_i)\leq \epsilon$. Let $\Delta(\F, d_e)$ be  diameter of the class $F$ under the distance metric $d_e$. 

Throughout the paper, we will focus on regression tasks and use the following statistical model between $Y$ and $X$ to demonstrate our results: 
\begin{align}\label{eqn:reg}
Y = h_*(X) + \varepsilon
\end{align}
where $\varepsilon\in\R$ is random noise independent of $X$, whose distribution is not necessarily sub-Gaussian. It is notable that the above model also capture some corruption models in $X$. For example, if $h_*(\x) = \w_*^{\top}\x$, then with  an additive corruption model $\xh = \x + \x_n$ we have $Y= \w_*^{\top}X + \w_*^{\top}X_n + \varepsilon = \w_*^{\top}X + \hat\varepsilon$. 

We consider the following definition of a truncation function. 
\begin{definition}\label{def}
A function $\phi_\alpha: \R_+ \to \R_+$ is a {\bf {truncation function}}  parameterized by $\alpha>0$ if (i) $\phi_\alpha(\cdot)$ is smooth, (ii) $\phi'_\alpha(x)=1$ if $x = 0$ and $\phi'_\alpha(x)=0$ if $x\rightarrow\infty$, (iii) $\phi'_\alpha(x)$   is a monotonically decreasing function, i.e., $\phi'_\alpha(x_1)\geq \phi'_\alpha(x_2)$ if $x_1\leq x_2$; (iv) there exists a universal constant $M>0$ such that $|\phi_\alpha(x) - x|\leq \frac{Mx^2}{\alpha}$, and for any $\alpha_1\leq \alpha_2$, we have $\phi_{\alpha_1}'(x)\leq \phi'_{\alpha_2}(x)$.
\end{definition}
According to the definition, we can see that $\phi''_\alpha(x)\leq 0$, which implies the non-convexity of $\phi_\alpha$.   The parameter $\alpha$ determines the truncation level, i.e., the larger the $\alpha$ the smaller the truncation. From (v) of the definition, we can see that when $\alpha=\infty$, we have $\phi_\alpha(x) = x$ meaning no truncation. Below, we will give some examples of truncation function.
\begin{figure*}[t]
\centering
{\includegraphics[scale=.25]{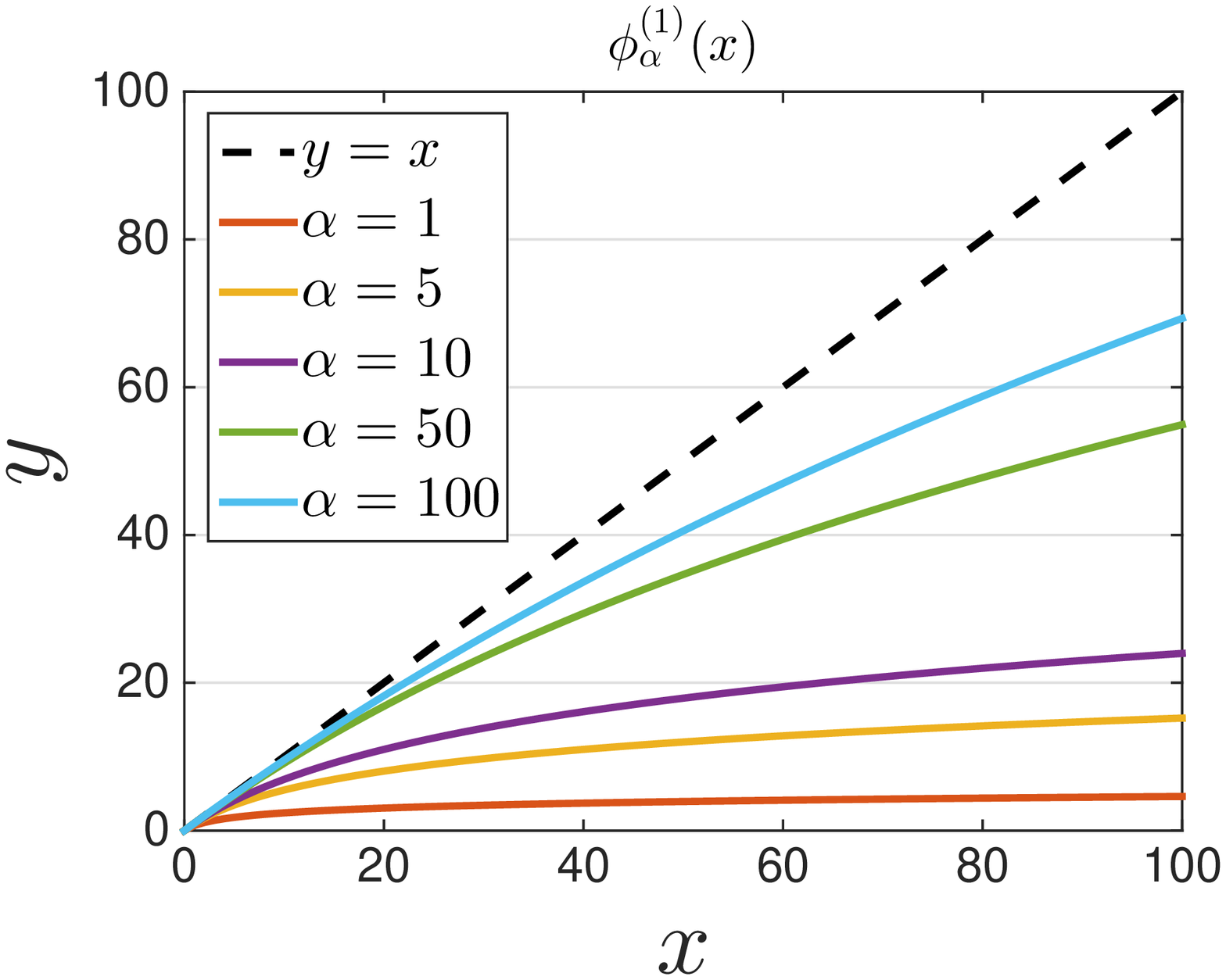}}
{\includegraphics[scale=.25]{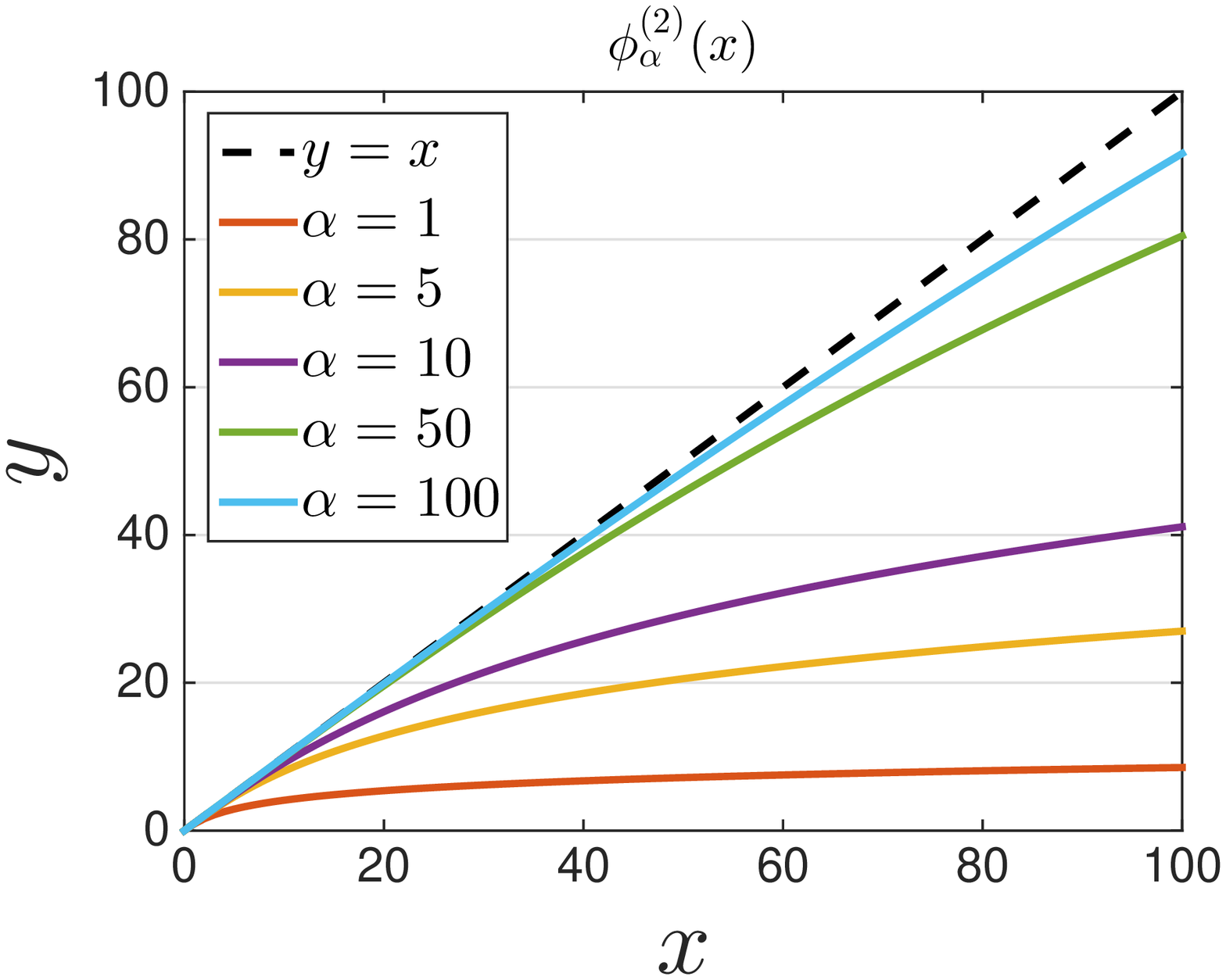}}
{\includegraphics[scale=.25]{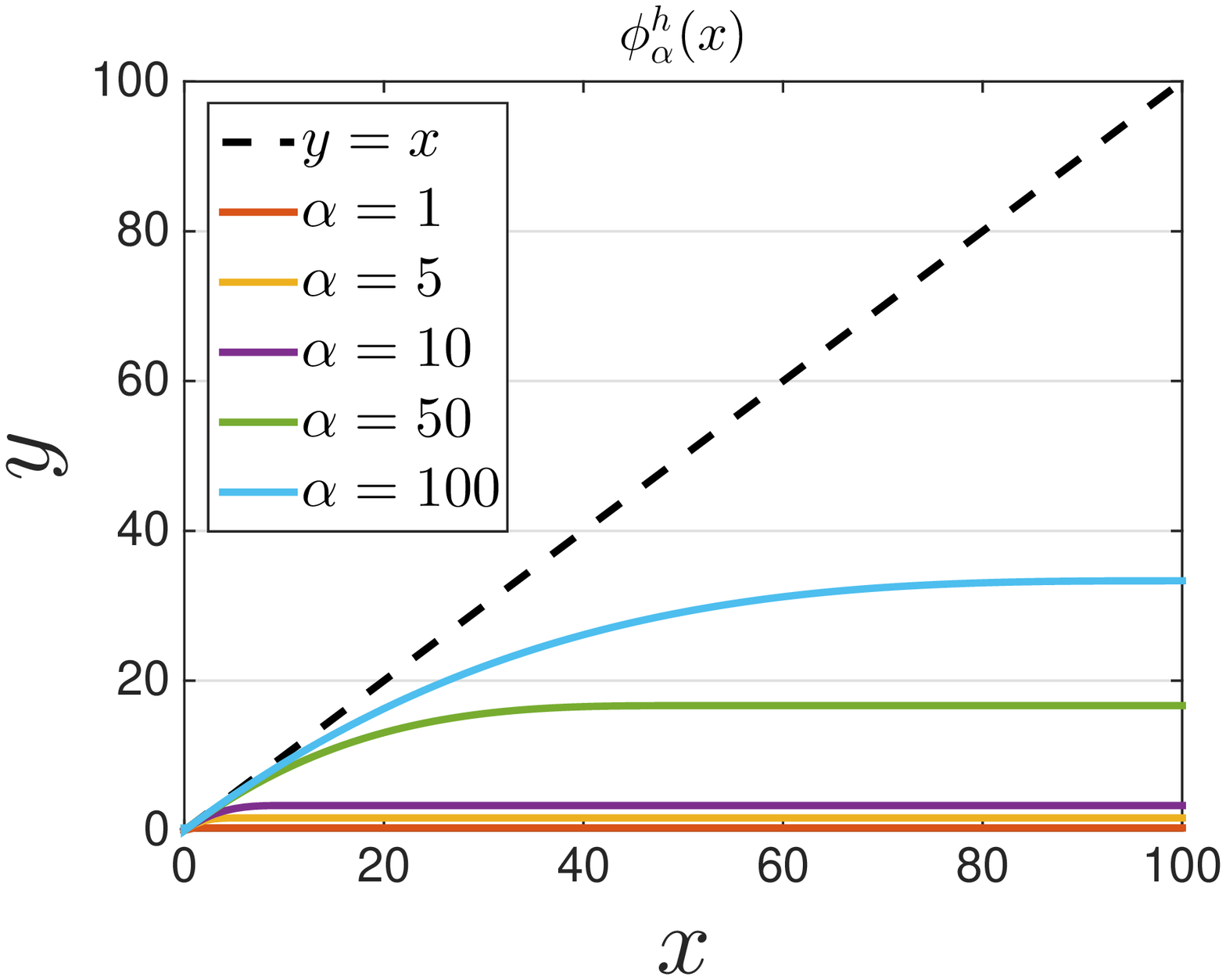}}
\caption{Visualization of different truncation losses with different  $\alpha$.}
\label{fig01}
\end{figure*}

{\bf Example 1}. $\phi^{(1)}_\alpha (x) = \alpha\log(1 + \frac{x}{\alpha})$. Applying this truncation to a square loss yields Cauchy loss for regression~\citep{Black1996}. We can verify that it is a truncation function and $|\phi_\alpha(x) - x|\leq \frac{x^2}{2\alpha}$ (see supplement).  \\
{\bf Example 2}. $\phi^{(2)}_\alpha (x) = \alpha\log(1 + \frac{x}{\alpha} + \frac{x^2}{2\alpha^2})$. This truncation has been considered by~\cite{brownlees2015empirical} for computing a mean estimator under heavy-tailed distribution of data. One could consider a more general function $\phi^{(m)}_\alpha (x) = \alpha\log(1 +\sum_{k=1}^{m} \frac{x^k}{\alpha^k k!})$. See supplement for verification of this function. \\
{\bf Example 3.} The following function can be shown to be a truncation function (see supplement): 
\begin{align*}
\phi^h_\alpha(x)=\left\{\begin{array}{ll}\frac{\alpha}{3}\left[1 - (1 - \frac{x}{\alpha})^3\right] & \text{ if } 0\leq x<\alpha\\ \frac{\alpha}{3}& \text{else}\end{array}\right.
\end{align*}
We plot the curves of the three  truncation functions with varying $\alpha$ in Figure~\ref{fig01}.

\subsection{Excess Risk Bounds of NCL with Truncated Losses}
This section concerns the excess risk bounds of NCL with truncated losses. Define: 
\begin{align}\label{eqn:em}
\fh = \arg\min_{f\in\F}P_n(\phi_\alpha(f)):=\frac{1}{n}\sum_{i=1}^n\phi_\alpha(f(Z_i)).
\end{align}

Our analysis and results in this section are based on the following assumption. 
\begin{ass}\label{ass:1}
There exists a constant $\sigma>0$ such that $\E[f(Z)^2] \leq \sigma^2$ for any $f\in\mathcal F$.
\end{ass}
{\bf Remark.} Please notice that the random function $f(Z)$ is not necessarily bounded, but it is reasonable to have a bounded mean and variance so that its second order moment is bounded. This assumption  also made in many previous works~\citep{brownlees2015empirical, hsu2016loss, bubeck2013bandits, cortes2013relative}. Next section will use a relaxed assumption for learning a linear model.  Below, we will use  the statistical model~(\ref{eqn:reg}) to demonstrate the above assumption could hold under heavy-tailed distribution of $Y$. 

\begin{thm}\label{thm1:heavy:soft:1}
Under Assumption~\ref{ass:1} and $\phi_\alpha(\cdot)$ is a truncation function,  
for any $\alpha>0$ with a probability at least $1-\delta$, we have
\begin{align*}
&P(\widehat f) - P(f^*)  \leq C\beta(\F, \alpha)\log(1/\delta)\left(\frac{\gamma_2(\F, d_e)}{\sqrt{n}} + \frac{\gamma_1(\F, d_m)}{n}\right) + \frac{2M\sigma^2}{\alpha}
\end{align*}
where $C$ is a universal constant, $M$ is a  constant appearing in Definition~\ref{def}, $\beta(\F, \alpha)\in(0,1]$ is a non-decreasing function of $\alpha$. 
\end{thm}
To understand the above result, we first present a corollary and an example below. 
\begin{cor}\label{cor:heavy:soft:1}
Under the same condition in Theorem~\ref{thm1:heavy:soft:1}, and $\ell(z, y)$ is a Lipschitz continuous function w.r.t the first argument and $\max_{X\in\X, h, h'\in\H}|h(X) - h'(X)|$ is bounded. 
By setting $\alpha\geq \Omega(\sqrt{n})$, with a probability at least $1-\delta$, we have
\begin{align*}
&P(\widehat f) - P(f^*)  \leq  O\left(\frac{\log(1/\delta)}{\sqrt{n}}\right).
\end{align*}
\end{cor}
Let us consider the statistical model~(\ref{eqn:reg}). To learn a predictive function, we can use an absolute loss function $\ell(z, y)=|z - y|$. By assuming that $\sup_{h\in\H, X\in\X}h(X)<\infty$ and $\E[Y^2]\leq \sigma^2$ (please note that the distribution of $Y$ or $\varepsilon$ could be heavy-tailed), then we have $\E[f(Z)^2]\leq 2\E[h(X)^2]+2\sigma^2$ and the conditions in Corollary~\ref{cor:heavy:soft:1} hold. As a result, the empirical minimizer $\fh$ of~(\ref{eqn:em}) with $\alpha\geq \Omega(\sqrt{n})$ has an excess risk bound of $\widetilde O(1/\sqrt{n})$. Other loss functions that are Lipchitz continuous for a regression problem include $\epsilon$-insensitive loss~\citep{rosasco2004loss}, piecewise-linear loss~\citep{koenker2005quantile}, and huber loss~\citep{huber1964robust}. In comparison, \cite{brownlees2015empirical} have derived a similar order of excess risk bound for Lipschitz continuous losses. However, their solution is based on solving a difficult problem~(\ref{eqn:catoni}), while our solution is empirical minimizer of the truncated losses. 

It is notable that the result in Theorem~\ref{thm1:heavy:soft:1} is restricted to Lipschitz continuous loss functions, which precludes some non-Lipschitz continuous loss functions for heavy-tailed data. One example is the square loss for regression $\ell(z, y) = (z-y)^2$. The reason for this restriction is that the analysis for Theorem~\ref{thm1:heavy:soft:1} hinges on the covering number of $\F$ under the metric $d_m$. Next, we present a result that relies on metrics $d_e$ and $d_s$, which could imply an $\widetilde O(1/\sqrt{n})$ excess risk bound of $\fh$ for square loss.

\begin{thm}\label{thm1:heavy:soft:2}
Under the same condition in Theorem~\ref{thm1:heavy:soft:1}, for any $\delta\in(0,1)$, let $\Gamma_\delta$ satisfy $\Pr(\gamma_2(\F, d_s)>\Gamma_\delta)\leq\delta/8$.  
 With a probability at least $1-3\delta$, we have
\begin{align*}
&P(\widehat f) - P(f^*)  \leq  C\beta(\F, \alpha)\max(\Gamma_\delta, \Delta(\F, d_e))\sqrt{\frac{\log(8/\delta)}{n}} + \frac{2M\sigma^2}{\alpha} 
\end{align*}
\end{thm}
{\bf Remark:} It is not difficult to see that the above result only uses distance metrics $d_e$ and $d_s$ of $\F$, which makes it possible to derive an $\widetilde O(1/\sqrt{n})$ excess risk bound of $\fh$ for least-squares regression without the Lipschitz continuous assumption.   

In particular, let us consider the regression model~(\ref{eqn:reg}) and assume that $\E[Y^4]\leq \sigma^2$ (heavy-tailed) and $\sup_{h\in\H, X\in\X}h(X)\leq \infty$. Let $\ell(h(X), Y)=(h(X)-Y)^2$. Then $\E[f(Z)^2]\leq 8\sigma^2 + 8\sup_{h\in\H, X\in\X}h(X)^4\triangleq \sigma^2_f$. By setting 
$\Gamma_\delta = 2\sqrt{2}\sqrt{\Delta^2(\H, d_m) + \E[Y^2] + \sqrt{\frac{8\sigma^2}{n\delta}}}\gamma_2(\H, d_m)$,
it was shown~\citep{brownlees2015empirical} that $\Pr(\gamma_2(\F, d_s)>\Gamma_\delta)\leq \delta/8$ and $\Gamma_\delta\geq \Delta(\F, d_e)$. By assuming  $\sup_{h\in\H, X\in\X}h(X)\leq \infty$, then $\Delta^2(\H, d_m)$ $\gamma_2(\H, d_m)$ are bounded.   As a result, Theorem~\ref{thm1:heavy:soft:2} implies an excess risk bound of $\widetilde O(1/\sqrt{n})$ for truncating  the square loss to learn $\fh$ with $\alpha>\sqrt{n}$. 

For comparison, we compare this result with that by \cite{audibert2011robust}, which focuses on learning a linear model with a square loss function. They  obtained an $\widetilde O(d/n)$ bound of regular ERM based on square losses for sufficiently large $n$, and also obtained $\widetilde O(1/n)$ bound for a non-convex min-max estimator. In contrast, our bound is worse but it is applicable to non-linear models and our formulation could enjoy faster solver, e.g., SGD. For linear models using a square loss function, in next section we will establish a stronger result than that by \cite{audibert2011robust}. 

Finally, we mention how the truncation level parameter $\alpha$ enters into the excess risk bounds in Theorems~\ref{thm1:heavy:soft:1}, \ref{thm1:heavy:soft:2}. In particular, let us compare learning with truncation and without truncation.  Indeed, $\beta(\F, \alpha)$ is related to Lipchitz constant of $\phi_\alpha(f)$ in terms of $f$. Without truncation $\alpha=\infty$, the first term in both bounds dominates and $\beta(\F, \alpha)=1$. With truncation (e.g., $\alpha\leq \infty$), the first term could be scaled down by $\beta(\F, \alpha)$, making it possible to lower  the overall bound.  However, it is difficult to quantify $\beta(\F, \alpha)$ due to that the analysis is based on a uniform bound for any $f\in\F$. To address this issue,  we will present a different analysis below to demonstrate the benefits of truncation.

\subsection{Statistical Error of  SGD for Learning a Linear Model with Truncated Square Loss }
One shortcoming of excess risk bound analysis in last section is that it is restricted to the empirical minimizers $\fh$ , which might not be obtained in practice due to the problem~(\ref{eqn:em}) is non-convex. It is well-known that non-convex problems could have bad local minimum or stationary points, and commonly used solvers (e.g., SGD) may stuck at local minimum and even stationary points. In this section, we provide a direct analysis of SGD for solving~(\ref{eqn:em}) to show that truncation has a clear advantage for reducing statistical error. It should be noted that it will be difficult to analyze SGD for a general problem~(\ref{eqn:em}). Instead, we consider a statistical model $y_i = \w_*^\top\x_i + \varepsilon_i,~~(i = 1,\dots, n)$,   and minimizing truncated square losses:
\begin{align}\label{opt:trun:ls}
\min_{\w \in \R^d} F_\alpha(\w)  = \frac{1}{2n}\sum_{i=1}^{n} \phi_\alpha\left((\w^\top\x_i - y_i)^2\right),
\end{align}
The update of SGD for minimizing~(\ref{opt:trun:ls}) is 
$\w_{t+1} = \w_t - \frac{\eta_t}{2} \nabla\phi_\alpha((\w_t^{\top}\x_{i_t} - y_{i_t})^2)$, 
where $i_t$ is a random sampled index.  Considering $\varepsilon_i$ is independent of $\x_i$, then $\w_*$ is the global minimizer of $\min_{\w\in\R^d}\E[(\w^{\top}\x_i - y_i)^2)]$. 
We first show that SGD can find an approximate stationary point of $F_\alpha(\w)$ with $O(1/\epsilon^4)$ iteration complexity. 
\begin{prop}\label{prop:1}
Assume $\phi_{\alpha}$ is a truncation function satisfying that there exists a constant $\kappa>0$ such that $|x^2\phi_\alpha''(x^2)|\leq \kappa$  for any $x$,  $\|\x_i\|_2\leq R$ and $\E[\|\nabla\phi_\alpha((\w_t^{\top}\x_{i_t} - y_{i_t})^2) - \nabla F_\alpha(\w_t)\|^2]\leq \sigma_\alpha$ for all $t=1,\ldots$. Then SGD finds an approximate stationary point $\E[\|\nabla F_\alpha(\w_\alpha)\|^2]\leq \epsilon^2$ with a complexity of $O(\sigma_\alpha^2/\epsilon^4)$. 
\end{prop}
{\bf Remark:} The condition $|x^2\phi_\alpha''(x^2)|\leq \kappa$ can be satisfied by the three examples presented before. The variance condition can be verified. Indeed, we can prove $\w_t$ reside in a bounded ball meaning this condition holds. In order to focus on our theme, we omit detailed discussion here. 

Next, we present a result showing the statistical error of an approximate stationary solution found by SGD that depends on the distribution of $\varepsilon$ for $\alpha<\infty$. For ease of understanding, we present a result for a particular truncation function. The result can be generalized to other truncation functions such that $|x^2\phi_\alpha''(x^2)|\leq \kappa$ as done in~\citep{DBLP:journals/corr/Loh15}. 
\begin{thm}\label{thm5}
  Suppose SGD returns an approximate  stationary point $\w_\alpha$ such that $\|\w_\alpha - \w_*\|_2\leq r$ and $\|\nabla F_\alpha(\w_\alpha)\|\leq \epsilon$. Assume $\x_i$ follows a sub-Gaussian distribution with parameter $\sigma_x^2$ and covariance matrix $\Sigma_x$, whose minimal eigen-value $\lambda_{min}(\Sigma_x)>0$,  $\phi_{\alpha}(x) = \alpha \log(1+x/\alpha)$,  $n\geq \Omega(d\log d)$ and the noise $\varepsilon_i$ follows a distribution such that 
\begin{align}\label{eqn:stationary}
c\sigma_x^2 (\Pr(\varepsilon_i^2\geq T^2/4)^{1/2} + \exp(-c' T^2/(2\sigma_x^2r^2)))\leq \frac{\lambda_{min}(\Sigma_x)}{20} 
\end{align}
for $T\leq \sqrt{\alpha}/2$,   then with high probability $1-c\exp(-c'\log d)$ we have
\begin{align}\label{eqn:st}
\|\w_\alpha - \w_*\|_2 \leq O\left(\sqrt{\frac{\alpha d\log d}{n}} + \frac{d\log d}{n}\frac{T^2}{r} + \epsilon\right)
\end{align}
\end{thm}
{\bf Remark:} The proof of above theorem builds on some results established in~\citep{DBLP:journals/corr/Loh15}. Here, we focus on new insights brought by the above results to justify truncation. 

First, it is notable that the noise $\varepsilon_i$ could be heavy-tailed. The condition~(\ref{eqn:stationary}) imposes a lower bound for $\alpha$ due to the constraint $T\leq \sqrt{\alpha}/2$ (i.e., truncation could not be arbitrarily large). An appropriate value should depend on the distribution of noise.  Within a certain truncation level, the statistical error bound in~(\ref{eqn:st}) implies that smaller $\alpha$ may yield a smaller error. 

Second, we show that the above result of an approximate stationary point to minimizing truncated square losses can achieve a similar order of statistical error as linear least-squares regression without truncation established by \cite{audibert2011robust} under similar assumptions. In particular, under the assumptions that $\E[\varepsilon_i^4]\leq \sigma$ and a boundedness assumption of  inputs, they achieved $F(\wh) - F(\w_*)\leq O(d/n)$, where $F$ is expected square loss, $\wh$ is the optimal empirical solution to minimizing square losses. Under an eigen-value condition $\lambda_{min}(\Sigma_x)>0$ as in above theorem it implies that $\|\wh - \w_*\|_2\leq O(\sqrt{d/n})$. In contrast,  assuming $\E[\varepsilon_i^4]\leq \sigma$, we have $\Pr(\varepsilon_i^2\geq T^2/4)\leq \frac{4\E[\varepsilon_i^2]}{T^2}\leq 4\sigma^2/T^2$ by Markov inequality. Therefore by choosing a large enough  $\alpha$ (e.g., $\alpha=\Theta(\max(1/\lambda_{\min}(\Sigma_x), \log(1/\lambda_{\min}(\Sigma_x))))$, we can make~(\ref{eqn:stationary}) holds by setting $T= \sqrt{\alpha}/2$. Then the statistical error bound of $\w_\alpha$ becomes $O(\sqrt{d\log d/n})$, which is comparable to $\wh$. We note that mismatch of the $\log d$ factor is caused by different assumptions on the inputs. Nevertheless,  $O(\sqrt{d\log d/n})$ is the minimax optimal rate when $\varepsilon_i$ follows a Gaussian distribution~\citep{DBLP:conf/nips/Bhatia0KK17}. 

Lastly, we show that the result in Theorem~\ref{thm5}  is  stronger than previous results on heavy-tailed noise (including Audibert and Cantoni's results), especially with large noise. In particular, we could let $\E[\varepsilon_i^k]$ (where $k\in 2\mathbb N^+$) grows as $n$. For example, assume that $\E[\varepsilon_i^k]=n^{c}$. Let us set $\alpha=n^{\beta}$ with $\beta<1$ and $T=n^{\beta/2}/2$. By Markov inequality,  we have $\Pr(\varepsilon_i^2\geq T^2/4)\leq O(\frac{\E[\varepsilon_i^k]}{T^k})\leq O(n^{c-\beta k/2})$ Assuming that $c < \beta k/2$  and $n$ is large enough, the inequality in~(\ref{eqn:stationary}) could hold. As a result, the statistical error becomes $O(\sqrt{d\log d/n^{1-\beta}})$, which still implies consistency of a stationary solution to minimizing the truncated losses. In contrast, most previous results on heavy-tailed noise assume  $\E[\varepsilon_i^k]$  is bounded by a constant~\citep{brownlees2015empirical, hsu2016loss, bubeck2013bandits, cortes2013relative,audibert2011robust}. 

\subsection{Finding Stationary Points of Truncated Losses by  SGD}
Finally, we briefly discuss the complexity of SGD for finding stationary points of  averaged truncated losses beyond the setting of square loss and linear model as in last section.  We assume the hypothesis is characterized by $\w$ and denote the loss function by $\ell(\w; \x, y)$ and the objective function in~(\ref{prob2}) becomes $F_\alpha(\w) = 1/n\sum_{i=1}^n\phi_\alpha(\ell(\w; \x_i, y_i))$. Note that $F_\alpha(\w)$ is non-convex due to that $\phi_\alpha$ is non-convex.  We consider two cases depending on whether $\ell(\w; \x, y)$ is a convex function or not.

The complexity of SGD has been extensively studied in literature, especially when $F_\alpha(\w)$ is smooth.  When $\ell(\w; \x, y)$ is a non-convex function of $\w$ (e.g., for learning deep neural networks), if it is a smooth and  Lipschitz continuous function of $\w$, then by the smoothness of $\phi_\alpha(\cdot)$ we can show that $F_\alpha(\w)$ is a smooth function with Lipschitz continuous gradient. Hence it can find a $\epsilon$-stationary point satisfying $\E[\|\nabla F_\alpha(\w)\|^2]\leq \epsilon^2$ with a complexity of $O(1/\epsilon^4)$~\citep{DBLP:journals/siamjo/GhadimiL13a}. 

If $\ell(\w; \x, y)$ is non-smooth and non-convex, characterizing the complexity of SGD becomes difficult, though is was shown that SGD can still converge to stationary points for a broad family of non-smooth non-convex functions~\citep{damektame}. 
Nevertheless, if $\ell(\w; \x, y)$ is a non-smooth convex function, e.g., for learning a linear model with absolute loss, $ \epsilon$-insensitive loss, piecewise linear loss, we can still characterize the complexity of SGD even without smoothness of the loss function.    It is notable that gradient of non-smooth non-convex function may not be defined at some points. However, we can define sub-differentiable of a non-smooth non-convex function $g(\w)$. Let $\partial g(\w)$ denote the sub-differentiable of $g(\w)$, which consists of a set of points $\v$ satisfying: 
\begin{align*}
g(\u)\geq g(\w) + \v^{\top}(\u - \w) + o(\|\u - \w\|), \text{as } \u\rightarrow \w
\end{align*}

For a convex  function $\ell(\w)$ and a smooth truncation function $\phi_\alpha(\ell)$,  we have $\partial \phi_\alpha(\ell(\w)) = \phi_\alpha'(\ell(\w))\partial \ell(\w)$. With this, we can define $\partial F_\alpha(\w)=1/n\sum_i\partial \phi_\alpha(\ell(\w^{\top}\x_i - y_i))$. A point $\w$ is said to be stationary point of $F_\alpha(\w)$ if $\text{dist}(0, \partial F_\alpha(\w))=0$, where $\text{dist}$ denotes the distance from a point to a set. For our problem, we can establish the following convergence result of SGD for minimizing $F_\alpha(\w)$ with Lipschitz continuous convex losses. 
\begin{prop}\label{prop2}
Assume $\ell(\w; \x, y)$ is convex and satisfies $\|\partial\ell(\w^{\top}\x_i -y_i)\|\leq G$, then SGD for minimizing $F_\alpha(\w)$ can find a point $\w_\alpha$  that is close to a point $\widetilde\w_\alpha$ such that $\E[\|\w_\alpha - \widetilde \w_\alpha\|^2_2]\leq \epsilon^2$, and $\E[\text{dist}(0, \partial F_\alpha(\widetilde\w_\alpha))^2]\leq \epsilon^2$ with a complexity $O(1/\epsilon^4)$. 
\end{prop}
{\bf Remark:} The result implies even for non-smooth loss functions, SGD for learning with truncated losses can converge to a point that is close to an approximate stationary point.  The idea for proving this result is that we prove $F_\alpha(\w)$ is a weakly convex function and  then the result of SGD for minimizing weakly convex function is applicable~\citep{davis2018stochastic}.

\section{Experiments}\label{sec:app}
We provide some empirical results to demonstrate the effectiveness of of the proposed approach for learning both linear and deep models. 
We use SGD to solve ERM with truncation and without truncation.  
A standard regularization term $\lambda\|\w\|^2$ is also added to ERM.  The values of $\lambda$ and $\alpha$ are selected by cross-validation. Two loss functions will be considered, namely absolute loss and square loss. The truncation function is $\phi^{(1)}_\alpha$.  Other truncation functions offer similar trend as reported results. 
\begin{figure*}[t] 
\begin{center}
{\includegraphics[scale=0.18]{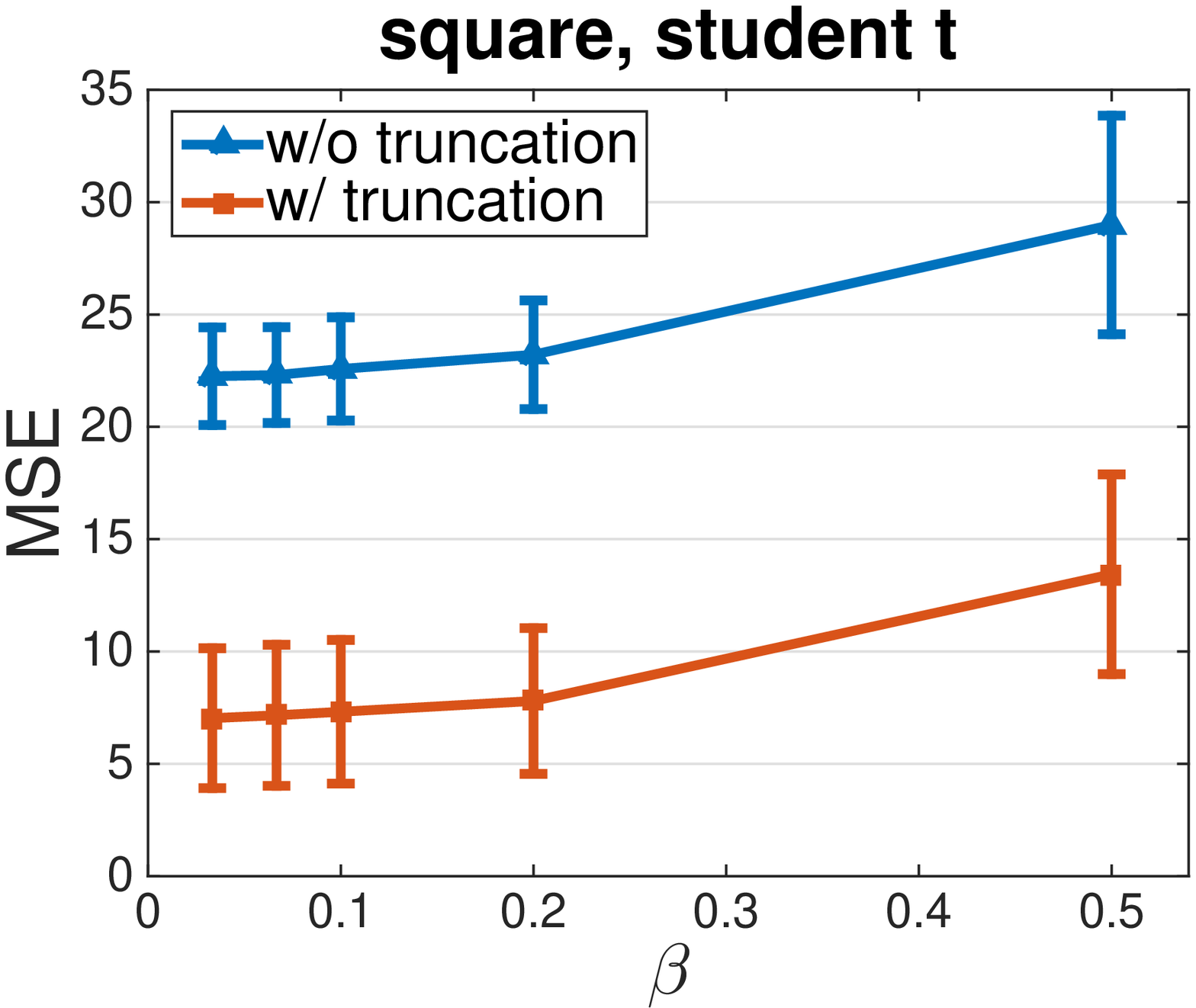}}
{\includegraphics[scale=0.18]{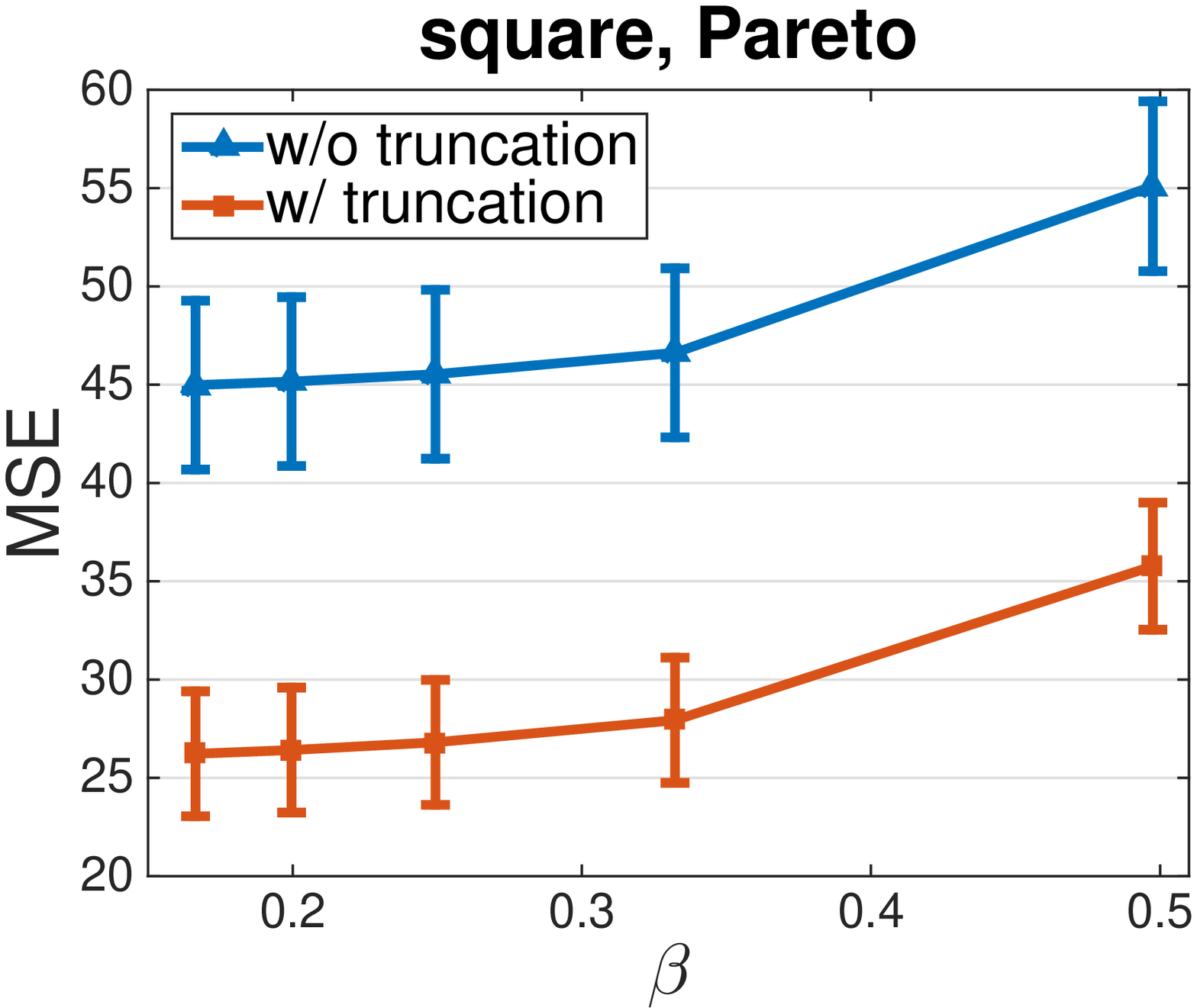}}
{\includegraphics[scale=0.18]{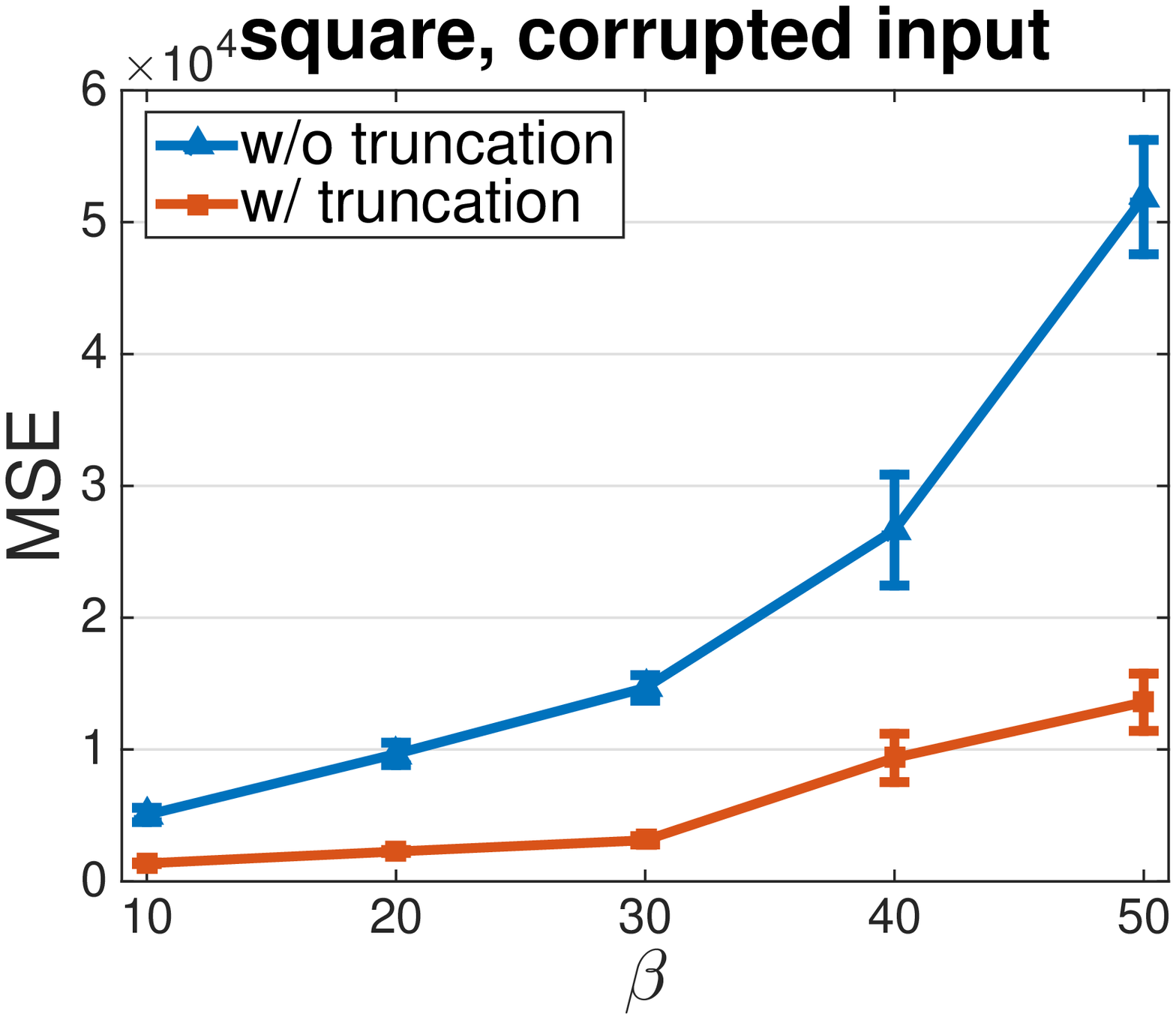}}
{\includegraphics[scale=0.18]{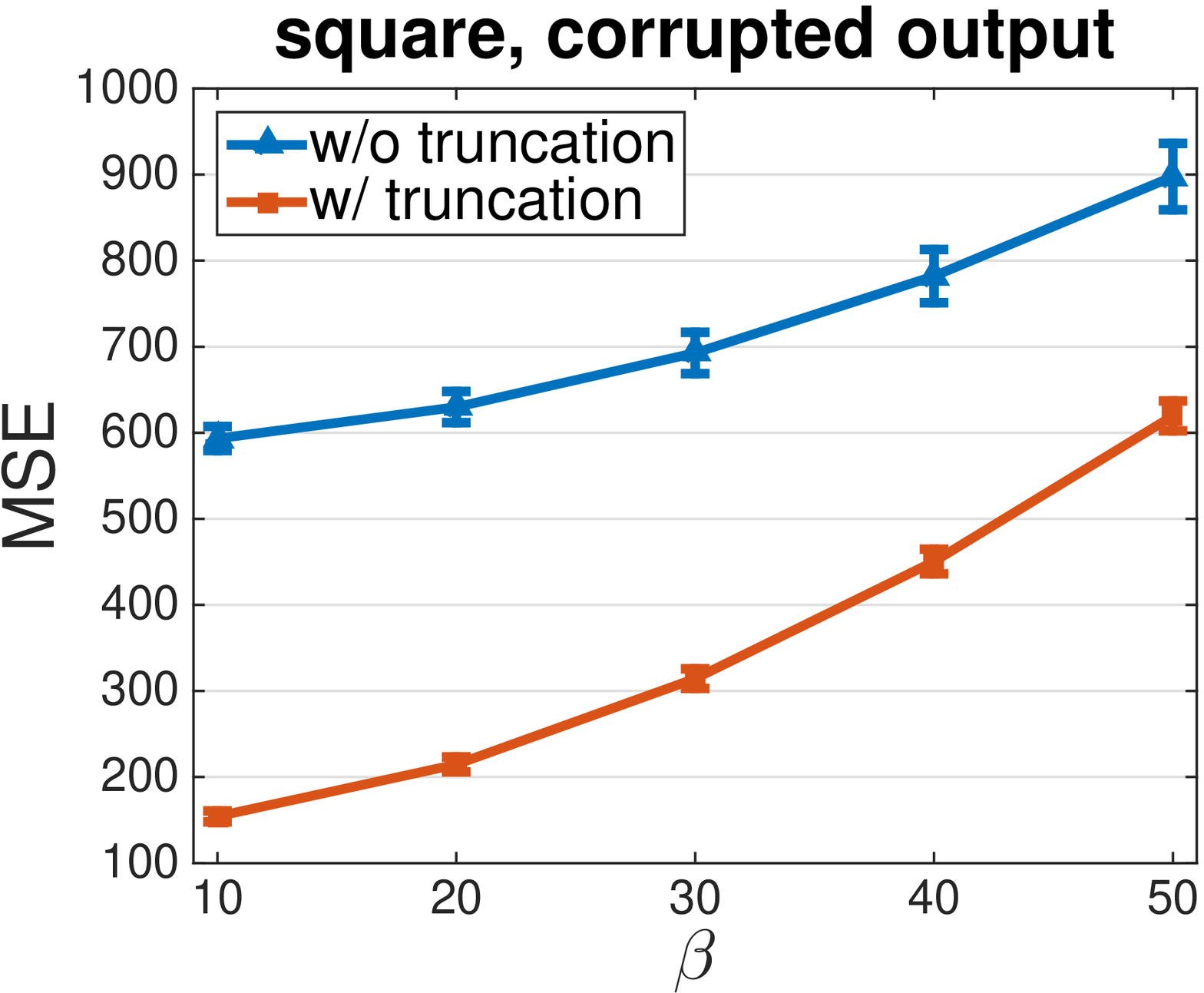}}
{ \includegraphics[scale=0.18]{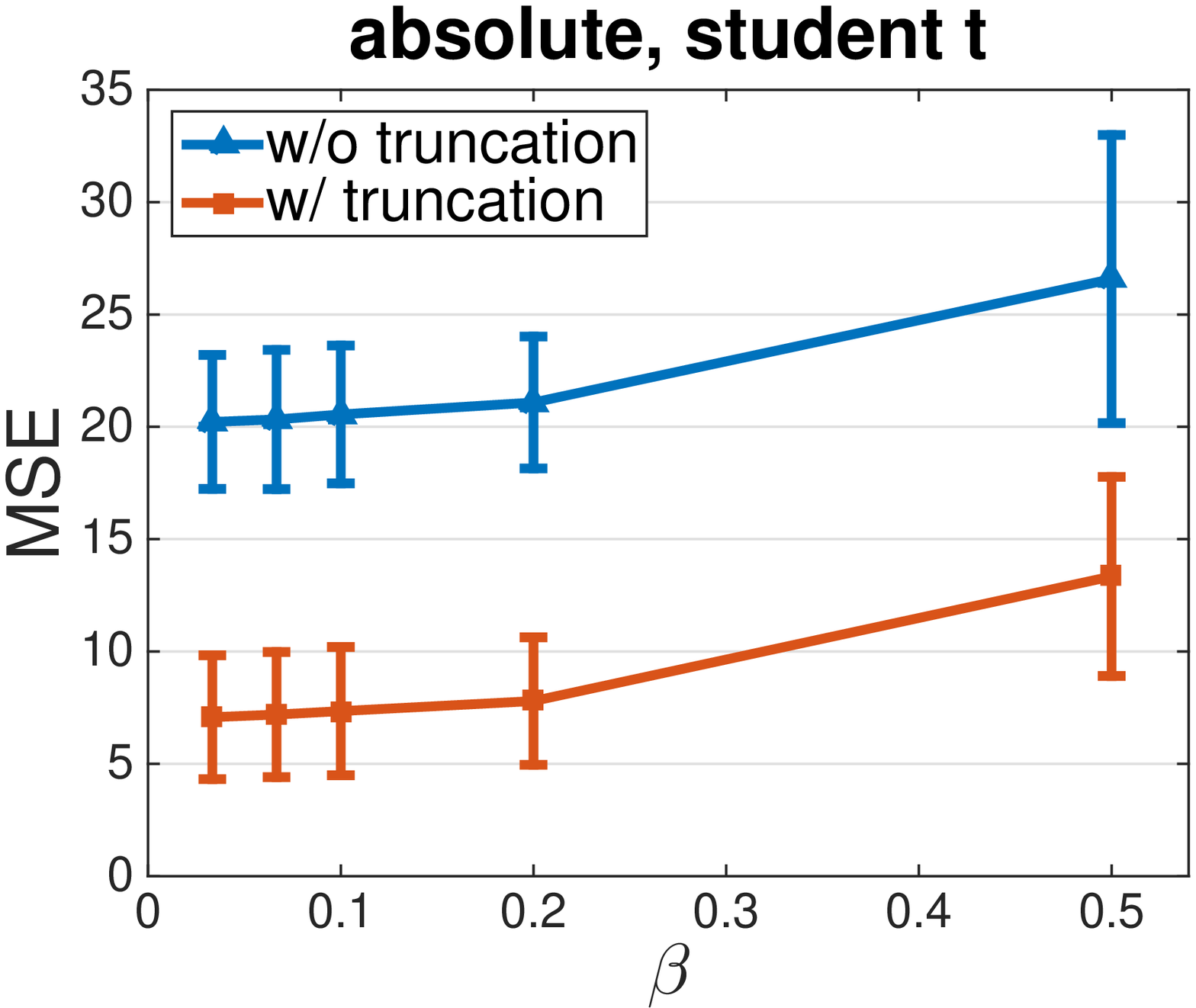}}
{ \includegraphics[scale=0.18]{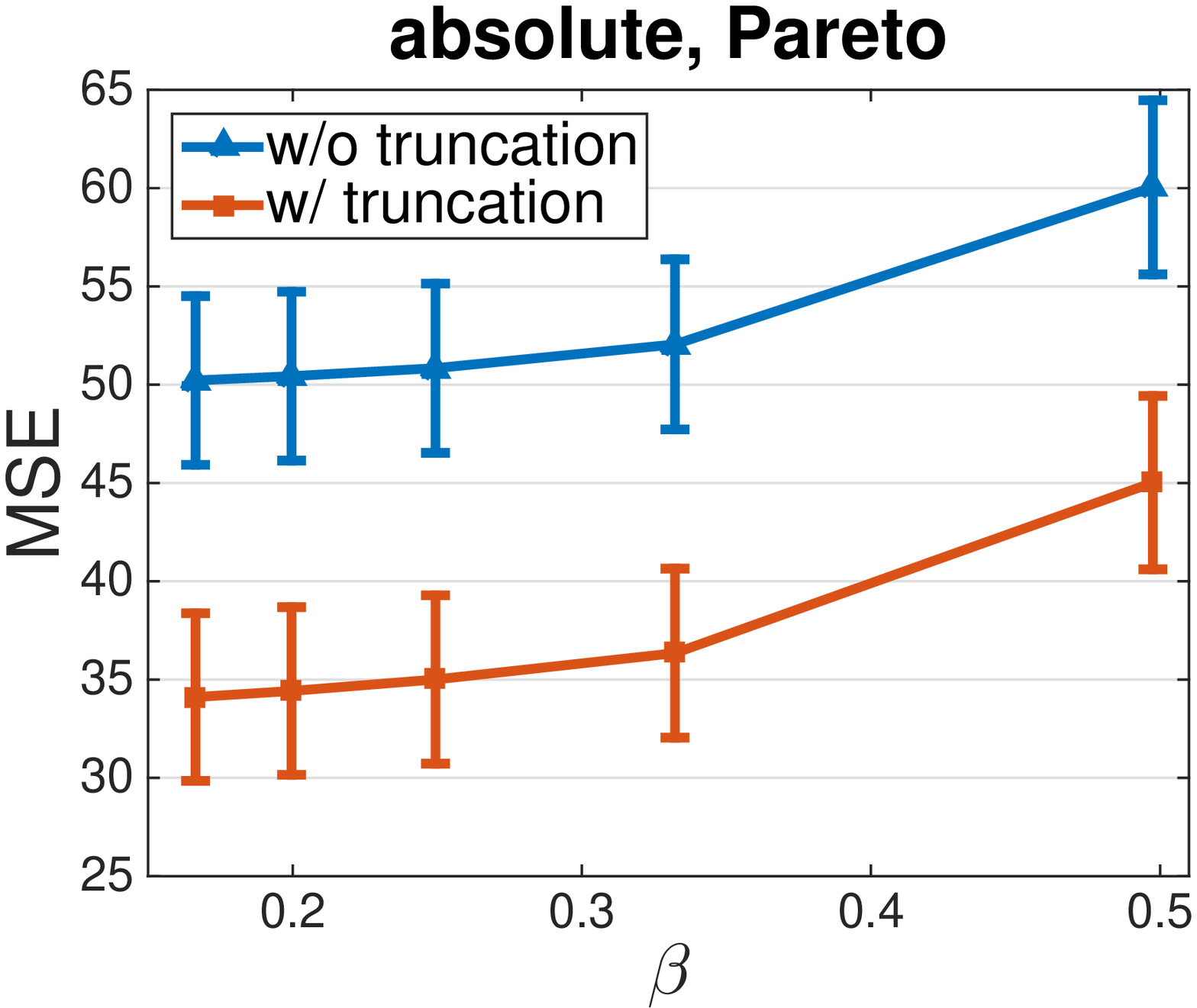}}
{ \includegraphics[scale=0.18]{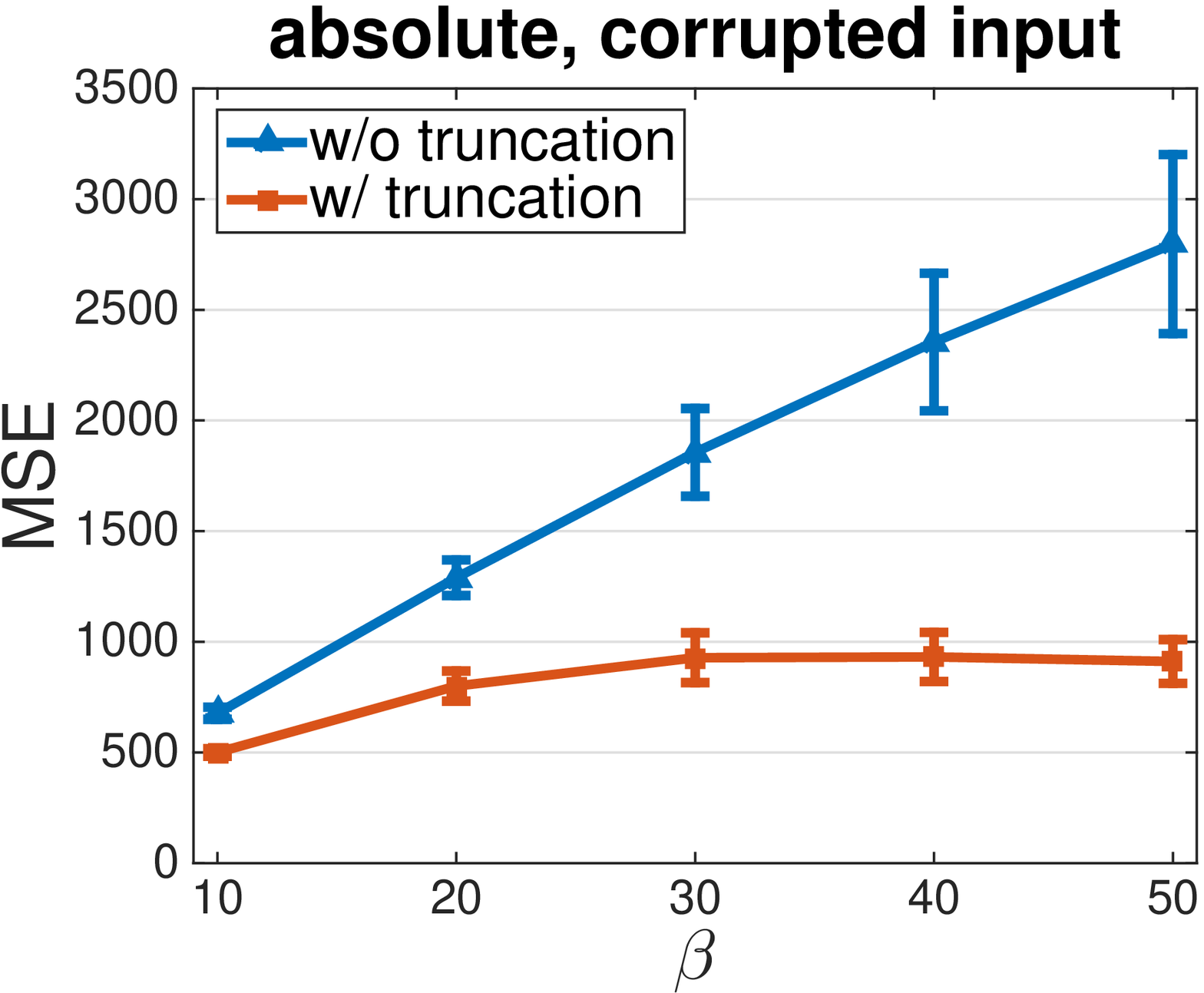}}
{ \includegraphics[scale=0.18]{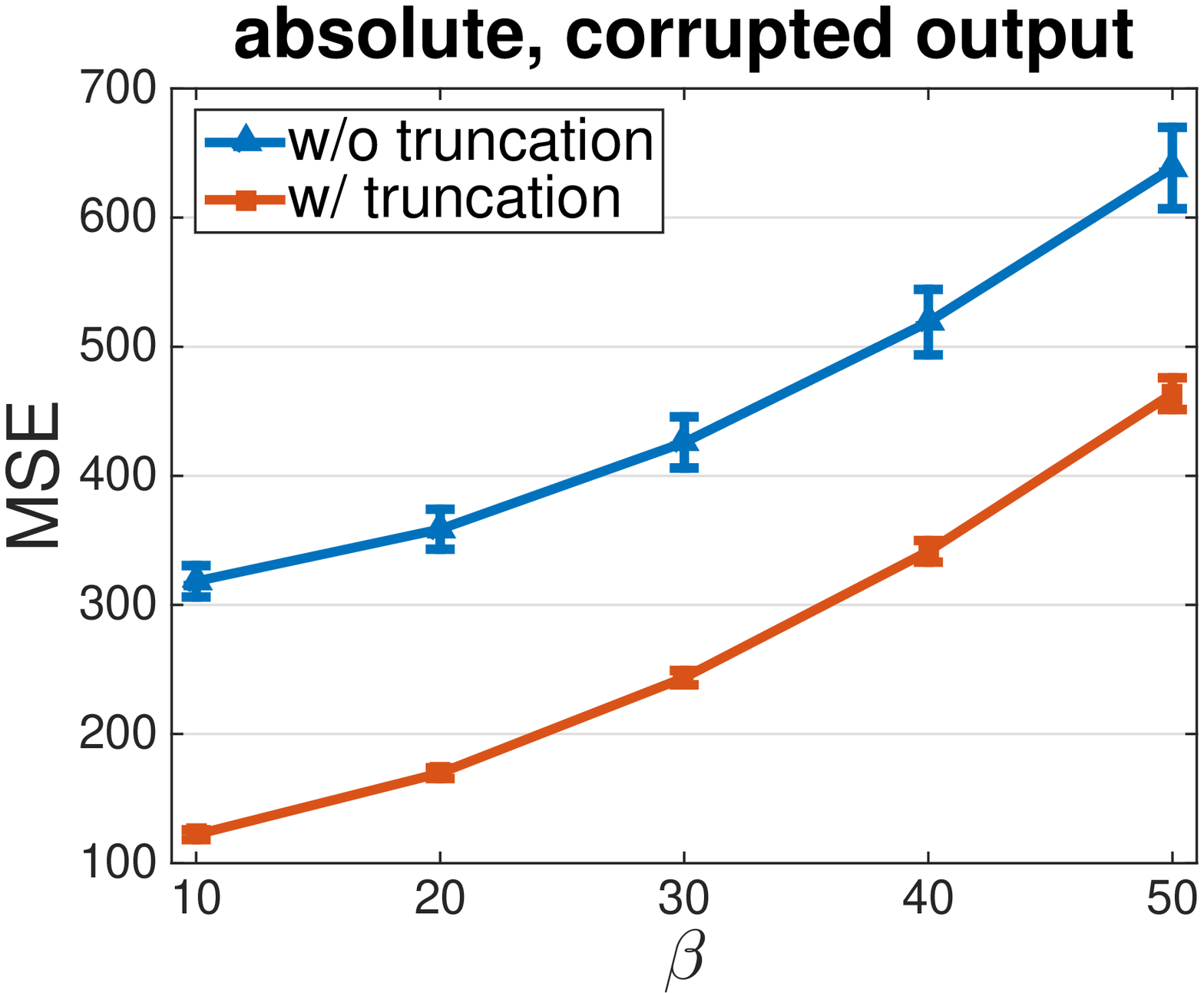}}
\caption{Comparisons of Testing Error for  w/ and w/o truncation with varying noise level.}
\label{fig:ls:abs}
\end{center}
\end{figure*}

{\bf Synthetic data.} 
 We conduct experiments on synthetic data first because it allows us to add different corruptions with varying noise level. We consider a linear regression model $y_i = \w_*^\top \x_i+ \varepsilon_i$, and two loss functions, i.e., square loss and absolute loss. 
We generate a random data matrix $X\in\R^{n_{\text{train}} \times d}$ with $n_{\text{train}} = 1000$ and $d=1000$. The entries of $X$ and $\w_*$ are generated independently with a standard Gaussian and a uniform distribution $U[0,1]$, respectively.
Then we add several types of noise into the statistical model for generating outputs. {\bf (a) student-$t$ noise} 
where the noise $\varepsilon$ follows a Student's $t$-distribution with degrees of freedom $1/\beta \in \{2, 5, 10, 15, 30\}$.
{\bf (b) Pareto noise} whre   
the noise $\varepsilon$ follows a Pareto distribution with tail parameter of $1/\beta \in \{2.01, 3.01, 4.01, 5.01, 6.01\}$, and then following by~\cite{brownlees2015empirical}, it is appropriately recentered in order to have zero mean.  {\bf (c) Corrupted output:} following by~\cite{bhatia2015robust}, a randomly generated sparse vector $\b$ is added to Gaussian noise $\varepsilon$ for generating $\y$.  The non-zero entries of $\b$ follow a uniform distribution $U[-\beta, \beta]$ with $\beta\in\{10, 20,30,40,50\}$. The sparsity is set to be $80\%$.
{\bf (d) Corrupted input:} following by~\cite{loh2012high},  $\x$ is corrupted by  $\z = \x + \x_\xi$ where $\x_\xi$ is independent of $\x$ and  follows a uniform distribution $U[-\beta, \beta]$ with $\beta\in\{10, 20,30,40,50\}$. Note that these corruptions have been considered in previous works and $\beta$ controls noise level in the corruption.  A testing dataset with the sample size of $n_{\text{test}}=1000$ is generated following the true model $y=\w_*^{\top}\x$ for evaluation. We report the testing mean-square-error (MSE) for different noise levels averaged over 5 random trials in Figure~\ref{fig:ls:abs}.
The results clearly show that the performance of learning with truncation by SGD are better than learning without truncation.  
\begin{table}[t]
\caption{Statistics of Datasets}\label{table:price:data}
\begin{center}
  \begin{tabular}{c | c| c |c| c| c }
    \hline
      \multicolumn{1}{l|}{data} &
      \multicolumn{1}{l|}{Pred.Date} &   
      \multicolumn{1}{l|}{Pred. Period} &  
      \multicolumn{1}{l|}{$n_{\text{train}}$} & 
      \multicolumn{1}{l|}{$n_{\text{test}}^\text{s}$}& 
      \multicolumn{1}{l}{$n_{\text{test}}^\text{i}$} \\
     \hline
P1 & 24, Apr & 02-08, May &588956 & 410 &	2689	\\
   \hline
P2 & 01, May& 09-15, May& 586761& 405 &  2523	\\
   \hline
P3 & 08, May & 16-22, May& 576386& 397 & 2561	\\
   \hline
P4 &  15, May & 23-29, May&564145 & 398 & 2775\\
   \hline
  \end{tabular}
\end{center}
\end{table}
\begin{table*}[t]
\caption{Comparison of Testing Error On Real Datasets}\label{table2:nonconvex:dnn:real:abs}
\begin{center}
  \begin{tabular}{c|c | c c| c c }
    \hline
     Model&\multicolumn{1}{c|}{Data} &\multicolumn{2}{c|}{Absolute  loss  (MAE)} &\multicolumn{2}{c}{Square loss (MSE) }   \\
    \hline
 &     \multicolumn{1}{l|}{\#} & 
           \multicolumn{1}{c}{w/o trunc.} & \multicolumn{1}{c|}{w/ trunc.} &  \multicolumn{1}{c}{w/o trunc.} & \multicolumn{1}{c}{w/ trunc.} \\
   \hline
linear model& house  & {6.8931} & {  5.1561}   & {66.4300} & {23.8871}  		\\

     \hline
&P1 & {  15.196} & {  14.113} & {  1482} & { 1167}    		\\
   \cline{2-6}
deep model& P2 &  {  17.766} &  {  16.797} &   {  2210} & {  1806} 		\\
   \cline{2-6}
(item-SKU level)&P3 &  {  22.104} &  {  18.642} & {  2375} & {  2049} 			\\
   \cline{2-6}
&P4 &  {  20.648} & {  14.176} &   {  2323} & {  1032}  				\\
\hline

&P1 & {   76.459} & {   74.190} &{  11726} &{  9515	}    		\\
   \cline{2-6}
deep model& P2 &  {  87.276} &  {   81.292} & {  38618} & {  15247} 		\\
   \cline{2-6}
(suppllier level)&P3 &  {  121.95} &  {  99.161} & {  28396} & {  17571}			\\
   \cline{2-6}
&P4 & {  137.06} & {  82.542} & {  33106} & {  11913}  				\\
\hline
  \end{tabular}  
\end{center}
\end{table*}

{\bf Real data.} We use a real dataset housing from libsvm website\footnote{\url{https://www.csie.ntu.edu.tw/~cjlin/libsvmtools/datasets/}} with sample size $n = 506$ to train a linear model.
We randomly select $n_\text{train} = 253$ as for training and cross-validation and the remaining as testing.
We also investigate a real-world application of learning deep neural networks  in e-commerce, and demonstrate that our theoretical results can be effectively applied to learning a deep non-linear model. The task is to forecast the weekly sales (e.g., future two weeks) of certain products and the statistics of datasets are shown in Tabel~\ref{table:price:data}). In online retailing, accurate forecasting is crucial since it helps the platform to design the promotion activities as well as online sellers to optimize the inventory strategies. A dataset of four continuous weeks in May 2017 is used for the experimental demonstration (denoted by P1$\sim$ P4). A total of 324 features including previous sales, consumer preference, and other useful information are collected. The statistics of each weekly data are included in supplement for reference. 
The DNN model has 5 layers, and ReLu is used as the activation function. In each hidden layer, the number of units is 80, and both input and output layers contain 50 units. For models learned with absolute  losses, mean-absolute-error (MAE) is used to measure the performance, while for model learned with square losses, MSE is used. The results are shown in Table~\ref{table2:nonconvex:dnn:real:abs}, which again demonstrate that the performance of learning with truncation has a significant improvement over that without truncation transformation for both linear and non-linear models. We also provide the Q-Q plots of the prediction error.
Q-Q plot is a graphical method for comparing two probability distributions by plotting their quantiles against each other~\citep{wilk1968probability}. If the two distributions are similar, the points in the Q-Q plot will approximately lie on the red line.
These results are presented in Figure~\ref{fig:QQ:house}, showing the heavy-tailed nature of the house data.
\begin{figure*}[t] 
\centering
	{\includegraphics[scale=0.2]{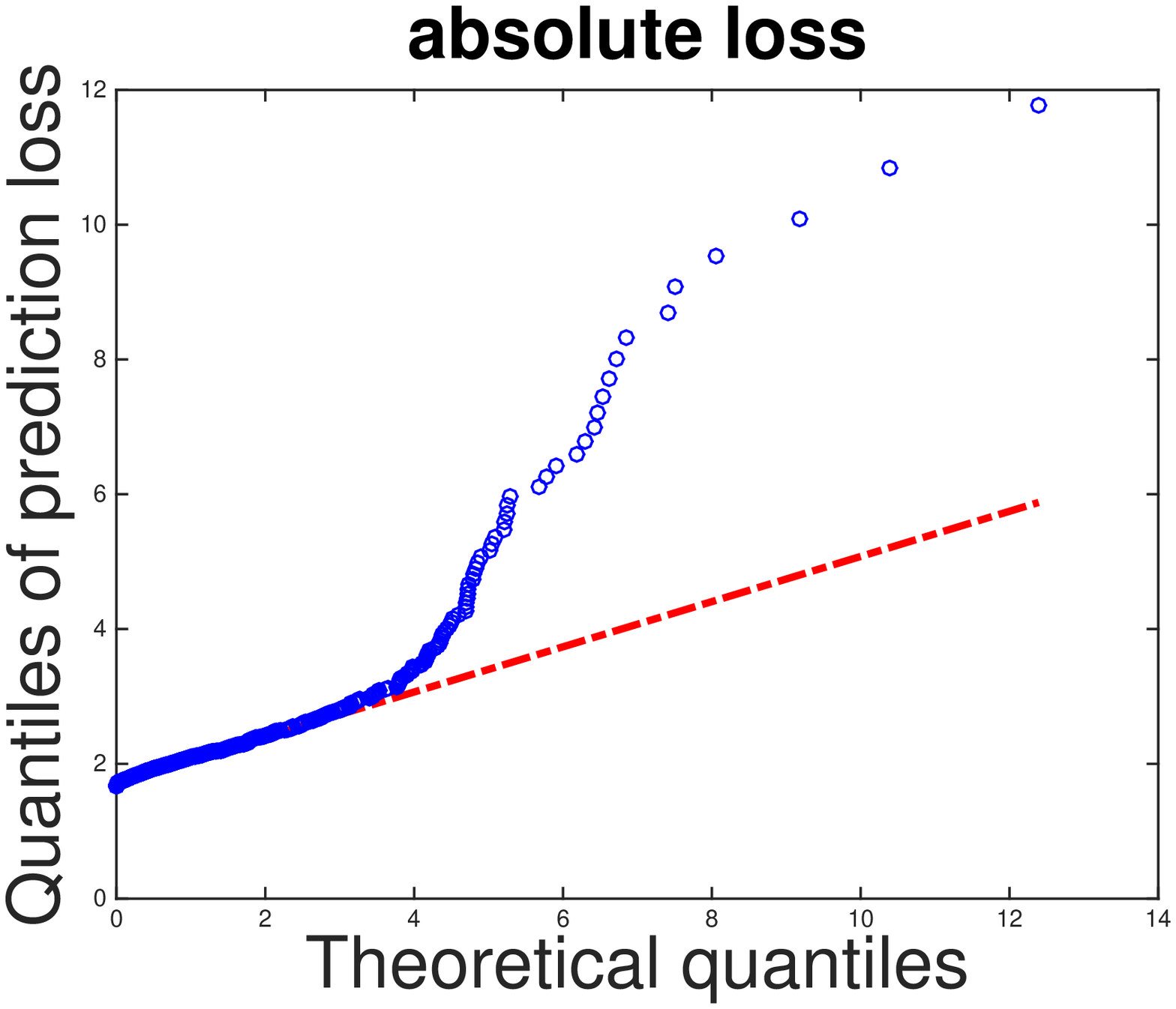} }
    {\includegraphics[scale=0.2]{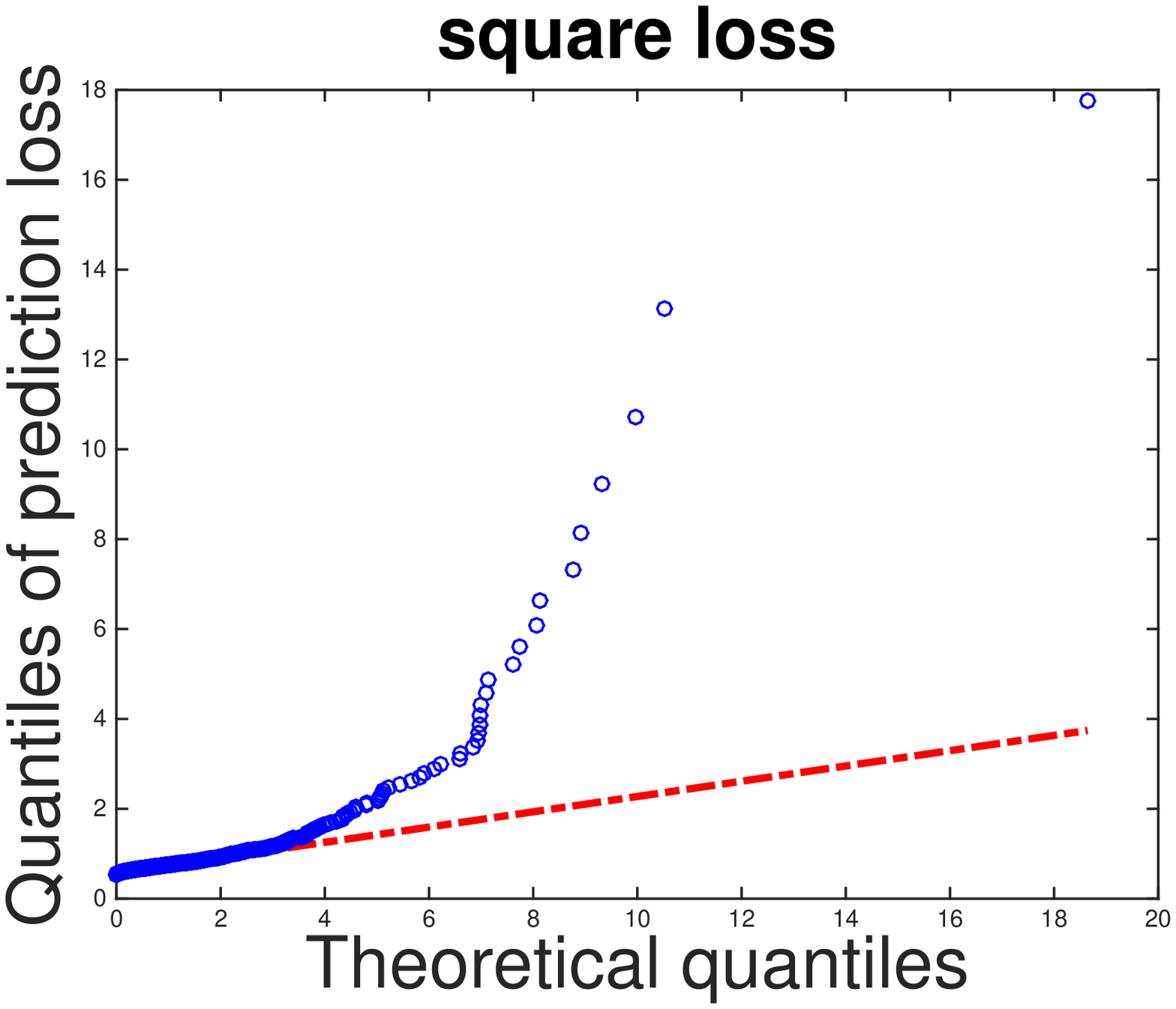}}
\caption{Q-Q plots on housing data}
\label{fig:QQ:house}
\end{figure*}

\section{Conclusions}\label{sec:con}
In this paper, we have considered non-convex learning with truncated losses from various perspectives and justified the benefit of truncation in the presence of large noise in data. For future work, we will consider analyze the statistical error of stationary points for other losses and develop stochastic algorithms for solving the involved problem with better time complexity.

\appendix
\section{Properties of truncation functions}
In this section, we first verify that three examples of trunction functions satisfy Definition 1. 


{\bf Example 1}. $\phi^{(1)}_\alpha (x) = \alpha\log(1 + \frac{x}{\alpha})$. We have
    $\phi'^{(1)}_{\alpha}(x) = \frac{1}{1+x/\alpha}$.
Then it is easy to check it satisfies condition (ii), (iii), and for any $\alpha_1\leq \alpha_2$, we have $\phi_{\alpha_1}'(x)\leq \phi'_{\alpha_2}(x)$. Since $\phi''^{(1)}_{\alpha}(x) = -\frac{1/\alpha}{(1+x/\alpha)^2}$, then $|\phi''^{(1)}_{\alpha}(x)| \leq  1/\alpha$, indicating that it satisfies condition (i).

{\bf Example 2}. $\phi^{(2)}_\alpha (x) = \alpha\log(1 + \frac{x}{\alpha} + \frac{x^2}{2\alpha^2})$. 
We have
    $\phi'^{(2)}_{\alpha}(x) = \frac{1+ \frac{x}{\alpha}}{1 + \frac{x}{\alpha} + \frac{x^2}{2\alpha^2}} =  1 -  \frac{1}{1+2\alpha/x + 2\alpha^2/x^2} $.
Then it is easy to check it satisfies condition (ii), (iii), and for any $\alpha_1\leq \alpha_2$, we have $\phi_{\alpha_1}'(x)\leq \phi'_{\alpha_2}(x)$. Since $\phi''^{(2)}_{\alpha}(x) = -\frac{1}{\alpha}\frac{\frac{x}{\alpha}+ \frac{x^2}{2\alpha^2}}{(1 + \frac{x}{\alpha} + \frac{x^2}{2\alpha^2})^2}$, then $|\phi''^{(2)}_{\alpha}(x)| \leq  1/\alpha$, indicating that it satisfies condition (i).

{\bf Example 3.} 
\begin{align*}
\phi^h_\alpha(x)=\left\{\begin{array}{ll}\frac{\alpha}{3}\left[1 - (1 - \frac{x}{\alpha})^3\right] & \text{ if } 0\leq x<\alpha\\ \frac{\alpha}{3}& \text{else}\end{array}\right.
\end{align*}
Then we have
\begin{align*}
\phi^{'h}_\alpha(x)=\left\{\begin{array}{ll}(1 - \frac{x}{\alpha})^2& \text{ if } 0\leq x<\alpha\\ 0& \text{else}\end{array}\right.
\end{align*}
Then it is easy to check it satisfies condition (ii), (iii), and for any $\alpha_1\leq \alpha_2$, we have $\phi_{\alpha_1}'(x)\leq \phi'_{\alpha_2}(x)$. Since 
\begin{align*}
\phi^{''h}_\alpha(x)=\left\{\begin{array}{ll}- \frac{2}{\alpha} (1 - \frac{x}{\alpha})& \text{ if } 0\leq x<\alpha\\ 0& \text{else}\end{array}\right.
\end{align*}
then $|\phi''^{h}_{\alpha}(x)| \leq  2/\alpha$, indicating that it satisfies condition (i).

Next, we will verify the condition $|x - \phi^{(2)}_\alpha(x)| \leq \frac{Mx^2}{\alpha}.$

\begin{prop}
For any $\alpha > 0$ and $x\ge 0$, we have
\begin{align}\label{prop1}
|x - \phi^{(1)}_\alpha(x)| \leq \frac{x^2}{2\alpha} ~\text{and}~|x - \phi^{(2)}_\alpha(x)| \leq \frac{x^2}{2\alpha}.
\end{align}
\end{prop}
\begin{proof}
We first need the following result to prove the proposition:
\begin{align}\label{prop1:trun1}
\exp(y) \geq 1 + y + \frac{y^2}{2}~\text{for all}~y\ge 0. 
\end{align}
Let's first condsider $ \phi^{(1)}_\alpha(x)$, to prove $|x - \alpha \log(1 + x/\alpha)| \leq \frac{1}{2\alpha} x^2$, we have to show $|x/\alpha - \log(1 + x/\alpha)| \leq \frac{1}{2\alpha^2} x^2$. Let $y = x/\alpha \ge 0$, we only need to show $|y -  \log(1 + y)| \leq \frac{y^2}{2}$.
By the inequality (\ref{prop1:trun1}) we know that $\log(1 + y) - y \leq 0$, so we only need to show $f(y) := y -  \log(1 + y) - \frac{y^2}{2}\leq 0$ for all $y\ge 0$. Since $f'(y) = -\frac{y^2}{1+y}\leq 0$, then we know $f(y)$ is a decreasing function on $y\ge 0$ thus $f(y) \leq f(0) = 0$, which give the first inequality in (\ref{prop1}).

Next let's consider $ \phi^{(2)}_\alpha(x)$. Similarly, we only need to show $f(y) := y -  \log(1 + y + y^2/2) - \frac{y^2}{2}\leq 0$ for all $y\ge 0$. Since $f'(y) = -\frac{y + y^2/2+y^3/2}{1+y+y^2/2}\leq 0$, then we know $f(y)$ is a decreasing function on $y\ge 0$ thus $f(y) \leq f(0) = 0$, which give the second inequality in (\ref{prop1}).
\end{proof}

\begin{prop}
For any $\alpha > 0$ and $x\ge 0$, we have
\begin{align}\label{prop1}
|x - \phi^{h}_\alpha(x)| \leq \frac{x^2}{\alpha},
\end{align}
\end{prop}
\begin{proof}
Let first consider $0 \leq x < \alpha$,  then we want to show $\left |x - \frac{\alpha}{3}[1-(1-\frac{x}{\alpha})^3]\right| \leq \frac{Mx^2}{\alpha}$, or equivalently $\left |\frac{x}{\alpha} - \frac{1}{3}[1-(1-\frac{x}{\alpha})^3]\right| \leq \frac{Mx^2}{\alpha^2}$. Let $y = \frac{x}{\alpha} \in[0,1)$, we only need to show  $\left |y - \frac{1}{3}[1-(1-y)^3]\right| \leq My^2$. 

(i) When $y - \frac{1}{3}[1-(1-y)^3]>0$, then we need to show $f(y):= y - \frac{1}{3}[1-(1-y)^3] - My^2 \leq 0$. In fact, $f'(y) = 1 - (1-y)^2 - 2My = 2(1-M)y - y^2$, By setting $M\ge1$, we know $f'(y)<0$.
Therefore, $f(y)\leq f(0) = 0$ for all $0\leq y < 1$. 

(ii) When $y - \frac{1}{3}[1-(1-y)^3]\leq 0$, then we need to show $f(y):=  \frac{1}{3}[1-(1-y)^3] - y - My^2 \leq 0$. In fact, $f'(y) =  (1-y)^2 - 1 - 2My =  -(1+2M)y - (1-y)y <0 $, then $f(y)\leq f(0) = 0$ for all $0\leq y < 1$. 

Next we consider $x\geq \alpha$, then we want to show $\left |x - \frac{\alpha}{3}\right| \leq \frac{Mx^2}{\alpha}$, or equivalently $\left |\frac{x}{\alpha} - \frac{1}{3}\right| \leq \frac{Mx^2}{\alpha^2}$.  Let $y = \frac{x}{\alpha} \geq 1 $, we only need to show  $\left |y - \frac{1}{3}\right| \leq My^2$. Since $y>1$, we must show  $y - \frac{1}{3} \leq My^2$. By setting $M \ge 1$, this  trivially holds. In summary, we can choose $M=1$, which complete the proof.
\end{proof}

\section{Proof of Theorem~\ref{thm1:heavy:soft:1}}
We will use the following lemma to prove this theorem. The proof of this lemma can be found later.
\begin{lemma}\label{lemma:concent:1}
Under the same setting as Theorem~\ref{thm1:heavy:soft:1}, with a probability at least $1-3\delta$, we have
\begin{align*}
&\sup_{f\in\F}|\Lambda(f) - \Lambda(f^*)| \leq C\beta(\F, \alpha)\log(2/\delta)\left(\frac{\gamma_2(\F, d_e)}{\sqrt{n}} + \frac{\gamma_1(\F, d_m)}{n}\right),
\end{align*}
where $\Lambda(f)  = P (\phi_\alpha(f))-P_n (\phi_\alpha(f))$, $C$ is a universal constant.
\end{lemma}
\begin{proof}[Proof of Theorem~\ref{thm1:heavy:soft:2}]
By (\ref{eqn:em}), we know $\widehat f = \arg\min_{f\in\mathcal F} P_n (\phi_\alpha(f))$, and thus $P_n (\phi_\alpha(\widehat f)) -P_n (\phi_\alpha(f^*)) \leq 0$, where $f^* = \arg\min_{f\in\mathcal F} P(f)$. Then we have
\begin{align*}
P(\widehat f )- P (f^*) =& [P(\widehat f) - P (\phi_\alpha(\widehat f))] + [P (\phi_\alpha(\widehat f)) - P_n (\phi_\alpha(\widehat f))]+ [P_n (\phi_\alpha(\widehat f)) -P_n (\phi_\alpha(f^*))]\\
&+ [P_n (\phi_\alpha(f^*))- P (\phi_\alpha(f^*))]+ [P (\phi_\alpha(f^*))- P(f^*)]\\
\leq&  [P(\widehat f) - P (\phi_\alpha(\widehat f))] + [P (\phi_\alpha(\widehat f)) - P_n (\phi_\alpha(\widehat f)) + [P_n (\phi_\alpha(f^*))- P (\phi_\alpha(f^*))] \\ &+ [P (\phi_\alpha(f^*))- P(f^*)]\\
\leq &  [P (\phi_\alpha(\widehat f)) - P_n (\phi_\alpha(\widehat f))] + [P_n (\phi_\alpha(f^*))- P (\phi_\alpha(f^*))] + \frac{2M\sigma^2}{\alpha}.
\end{align*}
where the last inequality is derived using the fact that $\E[|X-\phi_\alpha(X)|] \leq \E\left[\frac{M}{2\alpha}X^2\right]$ for a random variable $X$.
Then by Lemma~\ref{lemma:concent:1}, with a probability at least $1-3\delta$,
\begin{align*}
&P(\widehat f )- P (f^*) \leq C\beta(\F, \alpha)\log(2/\delta)\left(\frac{\gamma_2(\F, d_e)}{\sqrt{n}} + \frac{\gamma_1(\F, d_m)}{n}\right) + \frac{2M\sigma^2}{\alpha}.
\end{align*}
\end{proof}
\subsection{Proof of Lemma~\ref{lemma:concent:1}}
\begin{proof}
This proof is similar to the analysis in Proposition 5 and Lemma 6 from \citep{brownlees2015empirical}. For completeness, we include it here.
For any $f, f' \in \mathcal F$, we first know that $n(\Lambda(f) - \Lambda(f'))$ is the summation of the following independent random variables with zero mean:
\begin{align*}
C_i(f, f') = \phi_\alpha(f(Z_i)) -  \phi_\alpha(f'(Z_i)) - [\E[ \phi_\alpha(f(Z))] - \E[ \phi_\alpha(f'(Z))]] \leq 2 \beta(\F, \alpha) d_m(f,f'),
 \end{align*}
 where the last inequality is due to $\phi_\alpha$ is Lipschitz continuous and $\beta(\F, \alpha)=\sup_{f,Z}\phi_\alpha'(f(Z))$. On the other hand, 
 \begin{align*}
\sum_{i=1}^{n}\E[C_i(f, f')^2] \leq \sum_{i=1}^{n} \E[\phi_\alpha(f(Z_i)) -  \phi_\alpha(f'(Z_i))]^2 \leq  n \beta(\F, \alpha)^2 d_e(f,f')^2.
 \end{align*}
 Then by using Bernstein's inequality we have for any $f, f' \in \mathcal F$ and $\theta > 0 $, 
 \begin{align*}
 \text{Pr}(|\Lambda(f) - \Lambda(f')|> \theta) \leq 2\exp\left( - \frac{n \theta^2}{2(\beta(\F, \alpha)^2d_e(f,f')^2 + \theta \beta(\F, \alpha) d_m(f,f')/3 )} \right).
 \end{align*}
 Then by using Theorem 12 and inequality (14) from \citep{brownlees2015empirical}, let $f' = f^*$ we get
 \begin{align*}
&\sup_{f\in\F}|\Lambda(f) - \Lambda(f^*)| \leq C\beta(\F, \alpha)\log(2/\delta)\left(\frac{\gamma_2(\F, d_e)}{\sqrt{n}} + \frac{\gamma_1(\F, d_m)}{n}\right),
\end{align*}
where $C$ is a constant.
\end{proof}

\section{Proof of Corollary~\ref{cor:heavy:soft:1}}
\begin{proof}
By assumption we know that there exists a constant $D>0$ such that 
$\max_{X\in\X, h, h'\in\H}|h(X) - h'(X)| \leq D$. Then for any $X \in \X$, by the Lipschitz continuity  of $\ell$ function, we know that 
\begin{align*}
|\ell(h(X),Y) - \ell(h'(X),Y)| \leq  L |h(X) - h'(X)| \leq LD.
\end{align*}
where $L$ is the Lipschitz constant of $\ell()$ with respect to its first argument. 
By the definition of $\mathcal H$,
Since for any $f, f'\in\F$, we have $d_m(f, f') \leq L d_m(h, h')$, where $f=\ell(h(\cdot), \cdot)$ and $f' =\ell(h'(\cdot), \cdot)$. Hence an $\epsilon/L$-cover of $\mathcal H$ under the metric $d_m$ induces an $\epsilon$-cover of $\F$ under the metric $d_m$.  Therefore, we have
\begin{align*}
\log N(\mathcal F, \epsilon, d_m) \leq \log N(\mathcal H, \epsilon/L, d_m). 
\end{align*}
Since $\mathcal H$ is a compact set under distance measure $d_m$ by the assumption, its covering number is finite~\citep{cucker2002mathematical}. Then 
\begin{align*}
\gamma_1(\F, d_m)\leq \int_0^1{\log N(\F,\epsilon, d_m)}d\epsilon \leq \int_0^1{\log N(\mathcal H,\epsilon/L, d_m)}d\epsilon<\infty.
\end{align*}
Similarly, 
\begin{align*}
\gamma_2(\F, d_e)&\leq \int_0^1{\log N(\F,\epsilon, d_e)}^{1/2}d\epsilon \leq \int_0^1{\log N(\F,\epsilon, d_m)}^{1/2}d\epsilon \leq \int_0^1{\log N(\mathcal H,\epsilon/L, d_m)}^{1/2}d\epsilon\\
&\leq \infty
\end{align*}
By setting $\alpha\geq \Omega(\sqrt{n})$ in Theorem~\ref{thm1:heavy:soft:1}, we get the result.
\end{proof}

\section{Proof of Theorem~\ref{thm1:heavy:soft:2}}
We will use the following lemma to prove this theorem. The proof of this lemma can be found later.
\begin{lemma}\label{lemma:com:sup}
Under the same setting as Theorem~\ref{thm1:heavy:soft:2}, with a probability at least $1-3\delta$, we have
\begin{align*}
&\sup_{f\in\F}|\Lambda(f) - \Lambda(f^*)| \leq C\beta(\F, \alpha)\max(\Gamma_\delta, \Delta(\F, d_e))\sqrt{\frac{\log\left(\frac{8}{\delta}\right)}{n}},
\end{align*}
where $\Lambda(f)  = P (\phi_\alpha(f))-P_n (\phi_\alpha(f))]$, $C$ is a universal constant.
\end{lemma}
\begin{proof}[Proof of Theorem~\ref{thm1:heavy:soft:2}]
Similar to the proof of Theorem~\ref{thm1:heavy:soft:1}, we have
\begin{align*}
P(\widehat f )- P (f^*) 
\leq   [P (\phi_\alpha(\widehat f)) - P_n (\phi_\alpha(\widehat f))] + [P_n (\phi_\alpha(f^*))- P (\phi_\alpha(f^*))] + \frac{2M\sigma^2}{\alpha}.
\end{align*}
Then by Lemma~\ref{lemma:com:sup}, with a probability at least $1-3\delta$,
\begin{align*}
&P(\widehat f )- P (f^*) \leq  C\beta(\F, \alpha)\max(\Gamma_\delta, \Delta(\F, d_e))\sqrt{\frac{\log\left(\frac{8}{\delta}\right)}{n}} + \frac{2M\sigma^2}{\alpha}.
\end{align*}
Then by setting $\alpha \geq \sqrt{n \sigma^2/(2\log(1/\delta))}$, we get 
\begin{align*}
&P(\widehat f) - P(f^*)  \leq  O\left(\max(\Gamma_\delta, \Delta(\F, d_e))\sqrt{\frac{\log(8/\delta)}{n}}\right).
\end{align*}
\end{proof}

\subsection{Proof of Lemma~\ref{lemma:com:sup}}
\begin{proof}
This proof is similar to the analysis in Theorem 7 from \citep{brownlees2015empirical}. For completeness, we include it here.
First, we assume $\Gamma_\delta \geq \Delta(\F, d_e)$. Let $(Z'_1, \dots, Z'_n)$ be an independent copies of $(Z_1, \dots, Z_n)$, and we define
\begin{align*}
W_i(f) = \frac{1}{n} \phi_\alpha(f(Z_i)) - \frac{1}{n} \phi_\alpha(f(Z'_i)).
\end{align*}
For any $f\in\mathcal F$, we define
\begin{align*}
W(f) =\sum_{i=1}^{n}\varepsilon_iW_i(f),
\end{align*}
where $\varepsilon_1, \dots, \varepsilon_n$ are independent Rademacher random variables. Based on Hoeffding's inequality, we have for all $f, g \in \mathcal F$ and any $\theta > 0$,
\begin{align*}
\text{Pr}(|W(f)-W(g)|>\theta) \leq 2 \exp\left( -\frac{\theta^2}{2d_{s,s'}(f,g)}\right),
\end{align*}
where the probability is taken over Rademacher variables conditional on $Z_i$ and $Z'_i$, and $d_{s,s'}(f,g) = \sqrt{\sum_{i=1}^{n}(W_i(f)-W_i(g))^2}$.
Then by using Proposition 14 of \citep{brownlees2015empirical}, we have for all $\lambda > 0$, and a universal constant $C$
\begin{align}\label{prop14}
\E\left[\exp\left( \lambda \sup_{f\in\mathcal F} \left| W(f) - W(f^*)\right|\right)\right] \leq2 \exp\left( \lambda^2 C^2 \gamma(\mathcal F, d_{s,s'}(f,f^*))^2/4\right),
\end{align}
By the definition of $d_{s,s'}(f,g)$, we have
\begin{align*}
&d_{s,s'}(f,g) = \sqrt{\sum_{i=1}^{n}(W_i(f)-W_i(g))^2}\\
=&\left(\frac{1}{n^2}\sum_{i=1}^{n}[ \phi_\alpha(f(Z_i)) -  \phi_\alpha(f(Z'_i))-(\phi_\alpha(g(Z_i)) - \phi_\alpha(g(Z'_i)))]^2\right)^\frac{1}{2}\\
\leq &\frac{1}{n} \left(\sum_{i=1}^{n}[ \phi_\alpha(f(Z_i)) -  \phi_\alpha(g(Z_i))]^2\right)^\frac{1}{2}+\frac{1}{n} \left(\sum_{i=1}^{n}[\phi_\alpha(f(Z'_i)) - \phi_\alpha(g(Z'_i))]^2\right)^\frac{1}{2}\\
\leq &\frac{1}{\sqrt{n}} \beta(\F, \alpha)\left(\frac{1}{n} \sum_{i=1}^{n}[ f(Z_i) -  g(Z_i)]^2\right)^\frac{1}{2} +\frac{1}{\sqrt{n}}  \beta(\F, \alpha)\left(\frac{1}{n} \sum_{i=1}^{n}[f(Z'_i)- g(Z'_i)]^2\right)^\frac{1}{2},
\end{align*}
where the second inequality uses the fact that $\phi_\alpha(x)$ is Lipschitz continuous. Thus, we have
\begin{align}\label{thm4:ineq:1}
\gamma(\mathcal F, d_{s,s'}(f,g)) \leq \frac{1}{\sqrt{n}}\beta(\F, \alpha)  \gamma(\mathcal F, d_{s}(f,g)) + \frac{1}{\sqrt{n}}\beta(\F, \alpha) \gamma(\mathcal F, d_{s'}(f,g)).
\end{align}
Then we have
\begin{align*}
&\text{Pr}\left(\sup_{f\in\mathcal F}|W(f)-W(f^*)|\geq \theta\right) \\
\le&\text{Pr}\left(\sup_{f\in\mathcal F}|W(f)-W(f^*)|\geq \theta \left|\right. \gamma(\mathcal F, d_s) \leq \Gamma_\delta~~\text{and}~\gamma(\mathcal F, d_{s'}) \leq \Gamma_\delta \right) +2\text{Pr}\left(\gamma(\mathcal F, d_s) > \Gamma_\delta \right) \\
\leq & \E\left[\exp\left( \lambda \sup_{f\in\mathcal F} \left| W(f) - W(f^*)\right|\right)   \left|\right. \gamma(\mathcal F, d_s) \leq \Gamma_\delta~\text{and}~\gamma(\mathcal F, d_{s'}) \leq \Gamma_\delta \right] \exp(-\lambda\theta)+2\text{Pr}\left(\gamma(\mathcal F, d_s) > \Gamma_\delta \right) \\
\leq & 2 \exp\left( \lambda^2 C^2 \Gamma_\delta^2/n \right) \exp(-\lambda\theta) +2\text{Pr}\left(\gamma(\mathcal F, d_s) > \Gamma_\delta \right) \\
\leq & 2 \exp\left( \frac{\lambda^2 C^2 \Gamma_\delta^2}{n} -\lambda\theta \right) +\frac{\delta}{4}
\end{align*}
where the second inequality uses Markov inequality, the third inequality uses the results of (\ref{prop14}) and (\ref{thm4:ineq:1}), where the last inequality is due to the definition of $\Gamma_\delta$ which satisfies $\Pr(\gamma(\F, d_s)>\Gamma_\delta)\leq\delta/8$. 

Let $\theta = 2C \beta(\F, \alpha)\Gamma_\delta \sqrt{\frac{\log(8/\delta)}{n}}$ and $\lambda = \frac{\sqrt{n\log(8/\delta)}}{C \beta(\F, \alpha)\Gamma_\delta }$ then 
\begin{align*}
 \frac{\lambda^2 C^2 \Gamma_\delta^2}{n} -\lambda\theta  = \frac{\lambda^2 C^2 \Gamma_\delta^2}{n} - 2C\Gamma_\delta \sqrt{\frac{\log(8/\delta)}{n}}\lambda  =  -\log(8/\delta).
\end{align*}
Therefore,
\begin{align*}
&\text{Pr}\left(\sup_{f\in\mathcal F}|W(f)-W(f^*)|\geq \theta\right) \leq \frac{\delta}{4} +\frac{\delta}{4} = \frac{\delta}{2}
\end{align*}
By Lemma 3.3 from~\citep{geer2000applications}, we get
\begin{align}\label{lemma3.3}
\nonumber\text{Pr}\left(\sup_{f\in\mathcal F}|\Lambda(f)-\Lambda(f^*)|\geq 2\theta\right) \leq 2 \text{Pr}\left(\sup_{f\in\mathcal F}|W(f)-W(f^*)|\geq \theta\right) \leq \delta
\end{align}
and for any $f\in\mathcal F$, $\text{Pr}\left(\sup_{f\in\mathcal F}|\Lambda(f)-\Lambda(f^*)|\geq \theta\right)\leq \frac{1}{2}$.
On the other hand, by using $\E[\Lambda(f)-\Lambda(f^*)] =0 $ and Lipschitz continuous of $\phi_\alpha(x)$,  we have
\begin{align*}
&\frac{\text{Var}(\Lambda(f)-\Lambda(f^*))}{\theta^2} \leq \beta^2(\F, \alpha)\frac{\E[f(Z)-f^*(Z)]^2}{n \theta^2} 
\leq \beta^2(\F, \alpha)\frac{\Delta^2(\F, d_e)}{n \theta^2}.
\end{align*}
By applying Chebyshev's inequality, it suffices to get
\begin{align*}
\theta \geq \sqrt{2/n} \beta(\F, \alpha) \Delta(\F, d_e).
\end{align*}
If we assume $C>1$ and choose $\delta < 1/3$, then $C \beta(\F, \alpha)\Gamma_\delta  \sqrt{\frac{\log(8/\delta)}{n}} \geq \sqrt{2/n} \beta(\F, \alpha) \Delta^2(\F, d_e)$. Therefore,
we get
\begin{align*}
\text{Pr}\left(\sup_{f\in\mathcal F}|\Lambda(f)-\Lambda(f^*)|\geq 2C \beta(\F, \alpha)\Gamma_\delta  \sqrt{\frac{\log(8/\delta)}{n}}\right) \leq \delta
\end{align*}
We can get the similar result for $\Gamma_\delta < \Delta(\F, d_e)$ instead of $\Gamma_\delta$ by using the similar analysis. We then complete the proof.
\end{proof}

\section{Proof of Proposition~\ref{prop:1}}
\begin{proof}
Let define $z_i = \w^\top\x_i - y_i$, then $\nabla F_{\alpha}(\w) = \frac{1}{n}\sum_{i=1}^{n}\phi'_{(\alpha)}(z_i^2) z_i \x_i$ and $\nabla^2 F_{\alpha}(\w) = \frac{1}{n}\sum_{i=1}^{n}2\phi''_{(\alpha)}(z_i^2) z_i^2 \x_i\x_i^\top + \phi'_{(\alpha)}(z_i^2)\x_i\x_i^\top$. By the assumptions, there exists a constant $\kappa>0$, such that $\|\nabla^2 F_{\alpha}(\w)\|\leq (2\kappa+1)R^2$, indecating that $F_{\alpha}(\w)$ has a $(2\kappa+1)R^2$-Lipschitz continous gradient. Then we have
\begin{align*}
F_{\alpha}(\w_{t+1}) \leq& F_{\alpha}(\w_t) + \nabla F_{\alpha}(\w_t)^\top(\w_{t+1}-\w_t) + \frac{(2\kappa+1)R^2}{2}\|\w_{t+1}-\w_t\|^2\\
= & F_{\alpha}(\w_t) - \eta_t \nabla F_{\alpha}(\w_t)^\top\phi_\alpha((\w_t^\top\x_i-y_i)^2) + \frac{(2\kappa+1)R^2\eta_t^2}{2}\|\phi_\alpha((\w_t^\top\x_i-y_i)^2)\|^2\\
= & F_{\alpha}(\w_t) - \eta_t \nabla F_{\alpha}(\w_t)^\top\phi_\alpha((\w_t^\top\x_i-y_i)^2)\\& + \frac{(2\kappa+1)R^2\eta_t^2}{2}\|\nabla\phi_\alpha((\w_t^\top\x_i-y_i)^2)-\nabla F_{\alpha}(\w_t)+\nabla F_{\alpha}(\w_t)\|^2
\end{align*}
Taking expectation on both sides we have
\begin{align*}
\E[F_{\alpha}(\w_{t+1})-F_{\alpha}(\w_t)] \leq&  \frac{(2\kappa+1)R^2\eta_t^2-2\eta_t}{2}\E[\|\nabla F_{\alpha}(\w_t)\|^2] + \frac{(2\kappa+1)R^2\eta_t^2\sigma_\alpha}{2}\\
\leq&  -\frac{\eta_t}{2}\E[\|\nabla F_{\alpha}(\w_t)\|^2] + \frac{(2\kappa+1)R^2\eta_t^2\sigma_\alpha}{2},
\end{align*}
where the last inequality uses the fact that $\eta_t \leq \frac{1}{(2\kappa+1)L^2}$.
Summing up $t$ over $1,\dots,T$, we have
\begin{align}
  \sum_{t=1}^{T}\eta_t\E[\|\nabla F_{\alpha}(\w_t)\|^2] \leq 2(F_{\alpha}(\w_{1})-F_{\alpha}(\w_*)) +\sum_{t=1}^{T} (2\kappa+1)R^2\eta_t^2\sigma_{\alpha}.
\end{align}
By setting $\eta_t = \frac{1}{(2\kappa+1)R^2\sqrt{T}}$, we have
\begin{align}
  \E_{R}[\E[\|\nabla F_{\alpha}(\w_t)\|^2]] \leq \frac{2(2\kappa+1)R^2(F_{\alpha}(\w_{1})-F_{\alpha}(\w_*))}{\sqrt{T}} + \frac{\sigma_\alpha}{\sqrt{T}},
\end{align}
where $R$ is a uniform random variable supported on $\{1, \dots, T\}$. To achieve an approximate stationary point $\E[\|\nabla F_{\alpha} (\w_t)\|^2] \leq \epsilon^2$, the iteration complexity is $T=O(\sigma_{\alpha}^2/\epsilon^4)$. 
\end{proof}
{\bf Remark.} The condition of $\left|x^2\phi''_{\alpha}(x^2)\right|\leq \kappa$ for three different truncation functions presented in Preliminaries subsection can be easily checked. Example 1: $\left|x^2 \phi''^{{(1)}}_{\alpha}(x^2) \right| = \left|-\frac{x^2/\alpha}{(1+x^2/\alpha)^2}\right| \leq 1$; Example 2: $\left|x^2 \phi''^{{(2)}}_{\alpha}(x^2) \right| = \left|\frac{x^4/\alpha^2 + x^6/(2\alpha^3)}{(1+x^2/\alpha+x^4/(2\alpha^2))^2}\right| \leq 1$; Example 3: $\left|x^2 \phi''^{h}_{\alpha}(x^2) \right| = \left|\frac{2x^2(1-x^2/\alpha)}{\alpha}\right| \leq \alpha |1-\alpha|$ when $0\leq x\leq \alpha$, otherwise $\left|x^2 \phi''^{h}_{\alpha}(x^2) \right| =0$.

\section{Proof of Theorem~\ref{thm5}}
\begin{proof}
We will use the following lemma in our proof.
\begin{lemma}\citep{loh2017statistical}
Under the assumption of Theorem~\ref{thm5}, the following inequality holds for any $\w_1, \w_2\in\{\w: \|\w - \w_*\|_2\leq r\}$ with probability $1-c\exp(c'\log d)$,
\begin{align}\label{ineqSS0}
(\nabla F_\alpha(\w_1) - \nabla F_\alpha(\w_2))^\top (\w_1 - \w_2) \geq 
\frac{\alpha_T\lambda_{min}(\Sigma_x)}{16}\|\w_1 - \w_2\|_2^2 - \tau \frac{\log(d)}{n} \|\w_1 - \w_2\|_1^2,
\end{align}
where $\alpha_T := \min_{|u|\leq T} \phi_\alpha''(u)>0$, $\tau = \frac{C(\alpha_T + \kappa_2)^2\sigma_x^2T^2}{r^2}$, and $\kappa_2$ satisfies $\phi_\alpha''(u)\geq - \kappa_2$ for all $u$.
\end{lemma}
Then let's start our proof.
It is easy to show that $|\phi'_{\alpha}(u^2)| =|\frac{u}{1+ u^2/\alpha}| \leq \frac{\sqrt{\alpha}}{2}$ and $\phi''_{\alpha}(u^2) = \frac{1-u^2/\alpha}{(1+ u^2/\alpha)^2} \geq -\frac{1}{8}$, then $\kappa_2 = \frac{1}{8}$.
Let $T \leq \sqrt{\alpha}/2$, then $\alpha_T = \frac{12}{25}$.
Then 
\begin{align}\label{ineqSS0}
(\nabla F_\alpha(\w_\alpha) - \nabla F_\alpha(\w_*))^\top (\w_\alpha - \w_*) \geq a \|\w_\alpha - \w_*\|_2^2 - \tau \frac{\log(d)}{n} \|\w_\alpha - \w_*\|_1^2,
\end{align}
where $a =  \frac{3\lambda_{\min}(\Sigma_x)}{100}$ and $\tau = \frac{C\sigma_x^2 T^2}{r^2}$ and $C$ is a constant. 
Suppose SGD returns an approximate  stationary point $\w_\alpha$ such that $\|\w_\alpha - \w_*\|_2\leq r$ and $\|\nabla F_\alpha(\w_\alpha)\|_2\leq \epsilon$.  Since $\w_\alpha$ is a stationary point and $\w_*$ is feasible, we have
\begin{align}\label{ineqSS1}
\nabla F_\alpha (\w_\alpha)^\top (\w_* - \w_\alpha) \geq -\epsilon \|\w_* - \w_\alpha\|_2
\end{align}
By Proposition 1 of~\citep{loh2017statistical}, we have
\begin{align}\label{ineqSS2}
\nabla F_\alpha(\w_*)^\top (\w_\alpha - \w_*) \geq - c \frac{\sqrt{\alpha}}{2} \sigma_x \sqrt{\log(d)/n}\|\w_\alpha - \w_*\|_1
\end{align}
Combining inequalities (\ref{ineqSS0}) (\ref{ineqSS1}) and (\ref{ineqSS2}), we have
\begin{align*}
a \|\w_\alpha - \w_*\|_2^2 \leq &\epsilon \|\w_* - \w_\alpha\|_2 + c \frac{\sqrt{\alpha}}{2} \sigma_x \sqrt{\log(d)/n}\|\w_\alpha - \w_*\|_1+ \tau \frac{\log(d)}{n} \|\w_\alpha - \w_*\|_1^2\\
\leq &\epsilon \|\w_* - \w_\alpha\|_2 + c \frac{\sqrt{\alpha}}{2} \sigma_x \sqrt{d\log(d)/n}\|\w_\alpha - \w_*\|_2+ \tau \frac{d\log(d)}{n} \|\w_\alpha - \w_*\|_2^2\\
\leq &\epsilon \|\w_* - \w_\alpha\|_2 + c \frac{\sqrt{\alpha}}{2} \sigma_x \sqrt{d\log(d)/n}\|\w_\alpha - \w_*\|_2+ \tau r\frac{d\log(d)}{n} \|\w_\alpha - \w_*\|_2
\end{align*}
Then we get
\begin{align*}
\|\w_\alpha - \w_*\|_2 \leq O\left(\sqrt{\frac{\alpha d\log d}{n}} + \frac{T^2 d\log d}{rn} + \epsilon\right)
\end{align*}
\end{proof}

\section{Proof of Proposition~\ref{prop2}}
\begin{proof}
For similicity, let $\ell(\w) = \ell(\w; \x,\y)$.
By the defination of truncation function, we know that $\phi_\alpha(x)$ is smooth, i.e., for any $\w,\v\in\R^d$, there exists a constant $L_\alpha$ such that $\phi_\alpha(\ell(\v)) + \phi_\alpha'(\ell(\w))(\ell(\w)-\ell(\v)) - \frac{L_\alpha}{2}|\ell(\w)-\ell(\v)|^2 \leq \phi_\alpha(\ell(\w))$. Since $\ell$ is convex, i.e. for any $\w,\v\in\R^d$, $\ell(\w) \geq \ell(\v) + \partial \ell(\v)^\top (\w - \v)$, then
\begin{align*}
\phi_\alpha(\ell(\w)) - \phi_\alpha(\ell(\v)) 
\geq &  \phi_\alpha'(\ell(\w))\partial \ell(\v)^\top (\w - \v)  - \frac{L_\alpha}{2}|\ell(\w)-\ell(\v)|^2\\
\geq &  \phi_\alpha'(\ell(\w))\partial \ell(\v)^\top (\w - \v) - \frac{G^2L_\alpha}{2}\|\w-\v\|^2
\end{align*}
where the first inequality uses $\phi_\alpha'(\ell(\w))\geq 0$; the second inequality uses the fact that $\|\partial \ell(\w;\x_i,y_i)\| \leq G$. That is, $F_\alpha(\w)$ is $G^2L_\alpha$-weakly convex. Finally, by employing the result of Theorem 2.1 from~\citep{davis2018stochastic}, we can complete the proof.
\end{proof}

{\bibliographystyle{abbrv}
\bibliography{heavytailed,all}}

\begin{thebibliography}{51}
\providecommand{\natexlab}[1]{#1}
\providecommand{\url}[1]{\texttt{#1}}
\expandafter\ifx\csname urlstyle\endcsname\relax
  \providecommand{\doi}[1]{doi: #1}\else
  \providecommand{\doi}{doi: \begingroup \urlstyle{rm}\Url}\fi

\bibitem[Alon et~al.(1999)Alon, Matias, and Szegedy]{alon1999space}
Noga Alon, Yossi Matias, and Mario Szegedy.
\newblock The space complexity of approximating the frequency moments.
\newblock \emph{Journal of Computer and System Sciences}, 58:\penalty0
  137--147, 1999.

\bibitem[Audibert and Catoni(2011)]{audibert2011robust}
Jean-Yves Audibert and Olivier Catoni.
\newblock Robust linear least squares regression.
\newblock \emph{The Annals of Statistics}, 39\penalty0 (5):\penalty0
  2766--2794, 2011.

\bibitem[Bartlett and Mendelson(2006)]{bartlett2006empirical}
Peter~L Bartlett and Shahar Mendelson.
\newblock Empirical minimization.
\newblock \emph{Probability Theory and Related Fields}, 135\penalty0
  (3):\penalty0 311--334, 2006.

\bibitem[Belagiannis et~al.(2015)Belagiannis, Rupprecht, Carneiro, and
  Navab]{DBLP:conf/iccv/BelagiannisRCN15}
Vasileios Belagiannis, Christian Rupprecht, Gustavo Carneiro, and Nassir Navab.
\newblock Robust optimization for deep regression.
\newblock In \emph{2015 {IEEE} International Conference on Computer Vision
  (ICCV)}, pages 2830--2838, 2015.

\bibitem[Bhatia et~al.()Bhatia, Jain, and Kar]{DBLP:conf/nips/BhatiaJK15}
Kush Bhatia, Prateek Jain, and Purushottam Kar.
\newblock Robust regression via hard thresholding.
\newblock In \emph{Advances in Neural Information Processing Systems 28 (NIPS}.

\bibitem[Bhatia et~al.(2015)Bhatia, Jain, and Kar]{bhatia2015robust}
Kush Bhatia, Prateek Jain, and Purushottam Kar.
\newblock Robust regression via hard thresholding.
\newblock In \emph{Advances in Neural Information Processing Systems (NIPS)},
  pages 721--729, 2015.

\bibitem[Bhatia et~al.(2017)Bhatia, Jain, Kamalaruban, and
  Kar]{DBLP:conf/nips/Bhatia0KK17}
Kush Bhatia, Prateek Jain, Parameswaran Kamalaruban, and Purushottam Kar.
\newblock Consistent robust regression.
\newblock In \emph{{NIPS}}, pages 2107--2116, 2017.

\bibitem[Black and Anandan(1996)]{Black1996}
Michael~J Black and P~Anandan.
\newblock {The robust estimation of multiple motions: parametric and
  piecewise-smooth flow fields}.
\newblock \emph{Comput. Vis. Image Underst.}, 63:\penalty0 75--104, 1996.

\bibitem[Bottou(2010)]{bottou-2010-large}
L\'{e}on Bottou.
\newblock Large-scale machine learning with stochastic gradient descent.
\newblock In \emph{Proceedings of International Conference on Computational
  Statistics (COMPSTAT)}, pages 177--187, 2010.

\bibitem[Boucheron et~al.(2005)Boucheron, Bousquet, and
  Lugosi]{boucheron2005theory}
St{\'e}phane Boucheron, Olivier Bousquet, and G{\'a}bor Lugosi.
\newblock Theory of classification: A survey of some recent advances.
\newblock \emph{ESAIM: probability and statistics}, 9:\penalty0 323--375, 2005.

\bibitem[Brownlees et~al.(2015)Brownlees, Joly, and
  Lugosi]{brownlees2015empirical}
Christian Brownlees, Emilien Joly, and G{\'a}bor Lugosi.
\newblock Empirical risk minimization for heavy-tailed losses.
\newblock \emph{The Annals of Statistics}, 43\penalty0 (6):\penalty0
  2507--2536, 2015.

\bibitem[Bubeck et~al.(2013)Bubeck, Cesa-Bianchi, and
  Lugosi]{bubeck2013bandits}
S{\'e}bastien Bubeck, Nicolo Cesa-Bianchi, and G{\'a}bor Lugosi.
\newblock Bandits with heavy tail.
\newblock \emph{IEEE Transactions on Information Theory}, 59\penalty0
  (11):\penalty0 7711--7717, 2013.

\bibitem[Catoni(2012)]{catoni2012challenging}
Olivier Catoni.
\newblock Challenging the empirical mean and empirical variance: a deviation
  study.
\newblock In \emph{Annales de l'Institut Henri Poincar{\'e}, Probabilit{\'e}s
  et Statistiques}, volume~48, pages 1148--1185. Institut Henri Poincar{\'e},
  2012.

\bibitem[Chang et~al.(2018)Chang, Roberts, and
  Welsh]{doi:10.1080/00401706.2017.1305299}
Le~Chang, Steven Roberts, and Alan Welsh.
\newblock Robust lasso regression using tukey's biweight criterion.
\newblock \emph{Technometrics}, 60\penalty0 (1):\penalty0 36--47, 2018.

\bibitem[Cortes et~al.(2013)Cortes, Greenberg, and Mohri]{cortes2013relative}
Corinna Cortes, Spencer Greenberg, and Mehryar Mohri.
\newblock Relative deviation learning bounds and generalization with unbounded
  loss functions.
\newblock \emph{arXiv preprint arXiv:1310.5796}, 2013.

\bibitem[Cucker and Smale(2002)]{cucker2002mathematical}
Felipe Cucker and Steve Smale.
\newblock On the mathematical foundations of learning.
\newblock \emph{Bulletin of the American mathematical society}, 39\penalty0
  (1):\penalty0 1--49, 2002.

\bibitem[Dalalyan and Chen(2012)]{DBLP:conf/nips/DalalyanC12}
Arnak~S. Dalalyan and Yin Chen.
\newblock Fused sparsity and robust estimation for linear models with unknown
  variance.
\newblock In \emph{Advances in Neural Information Processing Systems (NIPS)},
  pages 1268--1276, 2012.

\bibitem[Davis and Drusvyatskiy(2018)]{davis2018stochastic}
Damek Davis and Dmitriy Drusvyatskiy.
\newblock Stochastic subgradient method converges at the rate $\text{O}
  (k^{-1/4})$ on weakly convex functions.
\newblock \emph{arXiv preprint arXiv:1802.02988}, 2018.

\bibitem[Davis et~al.(2018)Davis, Drusvyatskiy, Kakade, and Lee]{damektame}
Damek Davis, Dmitriy Drusvyatskiy, Sham Kakade, and Jason~D. Lee.
\newblock Stochastic subgradient method converges on tame functions.
\newblock \emph{CoRR}, 2018.

\bibitem[Dinh et~al.(2016)Dinh, Ho, Nguyen, and Nguyen]{dinh2016fast}
Vu~C Dinh, Lam~S Ho, Binh Nguyen, and Duy Nguyen.
\newblock Fast learning rates with heavy-tailed losses.
\newblock In \emph{Advances in Neural Information Processing Systems}, pages
  505--513, 2016.

\bibitem[Geer(2000)]{geer2000applications}
Sara~Anna Geer.
\newblock \emph{Applications of empirical process theory}.
\newblock Cambridge University Press, 2000.

\bibitem[Ghadimi and Lan(2013)]{DBLP:journals/siamjo/GhadimiL13a}
Saeed Ghadimi and Guanghui Lan.
\newblock Stochastic first- and zeroth-order methods for nonconvex stochastic
  programming.
\newblock \emph{{SIAM} Journal on Optimization}, 23\penalty0 (4):\penalty0
  2341--2368, 2013.

\bibitem[Gr{\"u}nwald and Mehta(2016)]{grunwald2016fast}
Peter~D Gr{\"u}nwald and Nishant~A Mehta.
\newblock Fast rates for general unbounded loss functions: from erm to
  generalized bayes.
\newblock \emph{arXiv preprint arXiv:1605.00252}, 2016.

\bibitem[Hsu and Sabato(2014)]{hsu2014heavy}
Daniel Hsu and Sivan Sabato.
\newblock Heavy-tailed regression with a generalized median-of-means.
\newblock In \emph{Proceedings of the 31st International Conference on Machine
  Learning (ICML-14)}, pages 37--45, 2014.

\bibitem[Hsu and Sabato(2016)]{hsu2016loss}
Daniel Hsu and Sivan Sabato.
\newblock Loss minimization and parameter estimation with heavy tails.
\newblock \emph{Journal of Machine Learning Research}, 17\penalty0
  (18):\penalty0 1--40, 2016.

\bibitem[Huber(1964)]{huber1964robust}
Peter~J Huber.
\newblock Robust estimation of a location parameter.
\newblock \emph{The annals of mathematical statistics}, pages 73--101, 1964.

\bibitem[Koenker(2005)]{koenker2005quantile}
Roger Koenker.
\newblock \emph{Quantile regression}.
\newblock Number~38. Cambridge university press, 2005.

\bibitem[Koltchinskii(2011)]{koltchinskii2011}
Vladimir Koltchinskii.
\newblock \emph{Oracle Inequalities in Empirical Risk Minimization and Sparse
  Recovery Problems}.
\newblock Springer, 2011.

\bibitem[Lecu{\'e} and Mendelson()]{lecue2013learning}
Guillaume Lecu{\'e} and Shahar Mendelson.
\newblock Learning subgaussian classes: upper and minimax bounds (2013).
\newblock \emph{Topics in Learning Theory-Societe Mathematique de France,(S.
  Boucheron and N. Vayatis Eds.)}.

\bibitem[Lecu{\'e} and Mendelson(2012)]{lecue2012general}
Guillaume Lecu{\'e} and Shahar Mendelson.
\newblock General nonexact oracle inequalities for classes with a
  subexponential envelope.
\newblock \emph{The Annals of Statistics}, 40\penalty0 (2):\penalty0 832--860,
  2012.

\bibitem[Lecu{\'e} and Mendelson(2017)]{lecue2016regularization2}
Guillaume Lecu{\'e} and Shahar Mendelson.
\newblock Regularization and the small-ball method ii: complexity dependent
  error rates.
\newblock \emph{The Journal of Machine Learning Research}, 18\penalty0
  (1):\penalty0 5356--5403, 2017.

\bibitem[Lecu{\'e} and Mendelson(2018)]{lecue2016regularization1}
Guillaume Lecu{\'e} and Shahar Mendelson.
\newblock Regularization and the small-ball method i: sparse recovery.
\newblock \emph{The Annals of Statistics}, 46\penalty0 (2):\penalty0 611--641,
  2018.

\bibitem[Liang et~al.(2015)Liang, Rakhlin, and Sridharan]{liang2015learning}
Tengyuan Liang, Alexander Rakhlin, and Karthik Sridharan.
\newblock Learning with square loss: Localization through offset rademacher
  complexity.
\newblock In \emph{Conference on Learning Theory}, pages 1260--1285, 2015.

\bibitem[Loh(2015)]{DBLP:journals/corr/Loh15}
Po{-}Ling Loh.
\newblock Statistical consistency and asymptotic normality for high-dimensional
  robust m-estimators.
\newblock \emph{CoRR}, abs/1501.00312, 2015.

\bibitem[Loh and Wainwright(2012{\natexlab{a}})]{DBLP:conf/isit/LohW12}
Po{-}Ling Loh and Martin~J. Wainwright.
\newblock Corrupted and missing predictors: Minimax bounds for high-dimensional
  linear regression.
\newblock In \emph{International Symposium on Information Theory (ISIT)}, pages
  2601--2605. {IEEE}, 2012{\natexlab{a}}.

\bibitem[Loh and Wainwright(2012{\natexlab{b}})]{loh2012high}
Po-Ling Loh and Martin~J Wainwright.
\newblock High-dimensional regression with noisy and missing data: Provable
  guarantees with nonconvexity.
\newblock \emph{The Annals of Statistics}, pages 1637--1664,
  2012{\natexlab{b}}.

\bibitem[Loh et~al.(2017)]{loh2017statistical}
Po-Ling Loh et~al.
\newblock Statistical consistency and asymptotic normality for high-dimensional
  robust $ m $-estimators.
\newblock \emph{The Annals of Statistics}, 45\penalty0 (2):\penalty0 866--896,
  2017.

\bibitem[Maronna et~al.(2006)Maronna, Martin, and Yohai]{citeulike:903734}
Ricardo~A. Maronna, Douglas~R. Martin, and Victor~J. Yohai.
\newblock \emph{Robust Statistics: Theory and Methods}.
\newblock John Wiley and Sons, New York, 2006.

\bibitem[Massart(2007)]{massart2007concentration}
Pascal Massart.
\newblock \emph{Concentration inequalities and model selection: Ecole d'Et{\'e}
  de Probabilit{\'e}s de Saint-Flour XXXIII-2003}.
\newblock Springer, 2007.

\bibitem[McWilliams et~al.(2014)McWilliams, Krummenacher, Lucic, and
  Buhmann]{DBLP:conf/nips/McWilliamsKLB14}
Brian McWilliams, Gabriel Krummenacher, Mario Lucic, and Joachim~M. Buhmann.
\newblock Fast and robust least squares estimation in corrupted linear models.
\newblock In \emph{Advances in Neural Information Processing Systems (NIPS)},
  pages 415--423, 2014.

\bibitem[Mehta and Williamson(2014)]{mehta2014stochastic}
Nishant~A Mehta and Robert~C Williamson.
\newblock From stochastic mixability to fast rates.
\newblock In \emph{Advances in Neural Information Processing Systems}, pages
  1197--1205, 2014.

\bibitem[Mendelson(2014)]{mendelson2014learning}
Shahar Mendelson.
\newblock Learning without concentration.
\newblock In \emph{Conference on Learning Theory}, pages 25--39, 2014.

\bibitem[Mendelson(2017)]{Mendelson2017}
Shahar Mendelson.
\newblock Learning without concentration for general loss functions.
\newblock \emph{Probability Theory and Related Fields}, 2017.
\newblock ISSN 1432-2064.
\newblock \doi{10.1007/s00440-017-0784-y}.

\bibitem[Nemirovsky and Yudin(1983)]{opac-b1091338}
Arkady~Semenovich Nemirovsky and David~Borisovich Yudin.
\newblock \emph{Problem complexity and method efficiency in optimization}.
\newblock Wiley-Interscience series in discrete mathematics. Wiley, Chichester,
  New York, 1983.
\newblock ISBN 0-471-10345-4.
\newblock A Wiley-Interscience publication.

\bibitem[Nguyen and Tran(2013{\natexlab{a}})]{DBLP:journals/tit/NguyenT13}
Nam~H. Nguyen and Trac~D. Tran.
\newblock Exact recoverability from dense corrupted observations via
  l1-minimization.
\newblock \emph{{IEEE} Trans. Information Theory}, 59\penalty0 (4):\penalty0
  2017--2035, 2013{\natexlab{a}}.

\bibitem[Nguyen and Tran(2013{\natexlab{b}})]{DBLP:journals/tit/NguyenT13a}
Nam~H. Nguyen and Trac~D. Tran.
\newblock Robust lasso with missing and grossly corrupted observations.
\newblock \emph{{IEEE} Trans. Information Theory}, 59\penalty0 (4):\penalty0
  2036--2058, 2013{\natexlab{b}}.

\bibitem[Rosasco et~al.(2004)Rosasco, Vito, Caponnetto, Piana, and
  Verri]{rosasco2004loss}
Lorenzo Rosasco, Ernesto~De Vito, Andrea Caponnetto, Michele Piana, and
  Alessandro Verri.
\newblock Are loss functions all the same?
\newblock \emph{Neural Computation}, 16\penalty0 (5):\penalty0 1063--1076,
  2004.

\bibitem[Talagrand(2005)]{talagrand2005generic}
M.~Talagrand.
\newblock \emph{The Generic Chaining: Upper and Lower Bounds of Stochastic
  Processes}.
\newblock Springer Monographs in Mathematics. Springer Berlin Heidelberg, 2005.
\newblock ISBN 9783540245186.
\newblock URL \url{https://books.google.com/books?id=dvn4hBt2QEoC}.

\bibitem[Van~de Geer(2000)]{van2000applications}
Sara~A Van~de Geer.
\newblock \emph{Applications of empirical process theory}, volume~91.
\newblock Cambridge University Press Cambridge, 2000.

\bibitem[Wilk and Gnanadesikan(1968)]{wilk1968probability}
Martin~B Wilk and Ram Gnanadesikan.
\newblock Probability plotting methods for the analysis for the analysis of
  data.
\newblock \emph{Biometrika}, 55\penalty0 (1):\penalty0 1--17, 1968.

\bibitem[Zhang et~al.(2017)Zhang, Yang, and Jin]{pmlr-v65-zhang17a}
Lijun Zhang, Tianbao Yang, and Rong Jin.
\newblock Empirical risk minimization for stochastic convex optimization:
  ${O}(1/n)$- and ${O}(1/n^2)$-type of risk bounds.
\newblock In \emph{Proceedings of the 2017 Conference on Learning Theory
  (COLT)}, volume~65, pages 1954--1979, 2017.

\end{thebibliography}

\end{document}